%% file: main.tex
\ifx\XeTeXversion\undefined\pdfoutput=1\fi
\documentclass[10pt,twocolumn,letterpaper]{article}

\usepackage[pagenumbers]{cvpr}
\input{preamble}

\title{Your Simulation Runs but Solves the Wrong Physics:\\
Verifying What LLM-Generated Simulation Code Encodes}

\author{
Zhenghan Song$^{1*}$ \quad
Yulong Liu$^{1*\dagger}$ \\
Cheng Wan$^{1}$ \quad
Chenjun Li$^{1}$ \quad
Lingfu Liu$^{1}$ \quad
Yunyi Li$^{2}$ \quad
Congcong Yuan$^{3,4}$ \\[4pt]
$^{1}$Cornell University \quad
$^{2}$Columbia University \\
$^{3}$Harvard University \quad
$^{4}$Nanyang Technological University
}

\begin{document}
\maketitle
\begingroup
\renewcommand\thefootnote{}\footnote{%
\noindent$^*$Joint first authors. Yulong Liu and Zhenghan Song contributed equally. The order of the joint first authors was determined by a coin flip.\\
$^\dagger$Correspondence to: Yulong Liu, \texttt{yl3825@cornell.edu}.}
\addtocounter{footnote}{-1}
\endgroup

\input{sec/0_abstract}
\input{sec/1_intro}
\input{sec/2_related}
\input{sec/3_moose}
\input{sec/4_method}
\input{sec/5_benchmark}
\input{sec/6_experiments}
\input{sec/7_discussion}
\input{sec/8_conclusion}
{
    \small
    \bibliographystyle{ieeenat_fullname}
    \bibliography{references}
}

\clearpage
\onecolumn
\appendix
\setcounter{table}{0}
\setcounter{figure}{0}
\setcounter{equation}{0}
\renewcommand{\thetable}{S\arabic{table}}
\renewcommand{\thefigure}{S\arabic{figure}}
\renewcommand{\theequation}{S\arabic{equation}}
\input{sec/X_suppl}

\end{document}

%% file: preamble.tex
\usepackage[utf8]{inputenc}
\usepackage[T1]{fontenc}
\usepackage{url}
\usepackage{amsthm}
\usepackage{amsfonts}
\usepackage{bm}
\usepackage{makecell}
\usepackage{nicefrac}
\usepackage{microtype}
\usepackage{needspace}
\usepackage{wrapfig}

\usepackage{listings}
\usepackage{tcolorbox}
\tcbuselibrary{skins, breakable}
\usepackage{float}
\usepackage{placeins}
\usepackage{tikz}
\usepackage{multirow}
\usepackage{longtable}
\usepackage{seqsplit}
\usepackage{adjustbox}
\ExplSyntaxOn
\NewDocumentCommand{\kernelname}{m}{
  \group_begin:
  \footnotesize\ttfamily
  \tl_set:Nn \l_tmpa_tl {#1}
  \regex_replace_all:nnN {([a-z0-9])([A-Z])} {\1\c{allowbreak}\2} \l_tmpa_tl
  \tl_use:N \l_tmpa_tl
  \group_end:
}
\ExplSyntaxOff

\usetikzlibrary{shapes, arrows.meta, positioning, shapes.symbols, fit, calc, decorations.pathreplacing}

\usepackage{algorithm}
\usepackage{algorithmic}

\definecolor{agent1color}{RGB}{230, 242, 255}
\definecolor{agent2color}{RGB}{240, 255, 240}
\definecolor{agent3color}{RGB}{255, 240, 240}
\definecolor{codebg}{RGB}{245, 245, 245}

\lstdefinelanguage{json}{
    basicstyle=\ttfamily\small,
    showstringspaces=false,
    breaklines=true,
    frame=single,
    backgroundcolor=\color{codebg},
    literate=
     *{:}{{{\color{red!70!black}{:}}}}{1}
      {,}{{{\color{red!70!black}{,}}}}{1}
      {\{}{{{\color{blue}{\{}}}}{1}
      {\}}{{{\color{blue}{\}}}}}{1}
      {[}{{{\color{blue}{[}}}}{1}
      {]}{{{\color{blue}{]}}}}{1},
}

\lstdefinelanguage{moose}{
    basicstyle=\ttfamily\small,
    backgroundcolor=\color{codebg},
    frame=single,
    breaklines=true,
    keywordstyle=\color{blue}\bfseries,
    morekeywords={Mesh, Variables, AuxVariables, Kernels, BCs, Executioner, Outputs, Materials},
    commentstyle=\color{green!50!black},
    stringstyle=\color{red!70!black},
    morecomment=[l]{\#},
    alsoletter={[]}
}

\newtcolorbox{promptbox}[2][]{
  colback=white,
  colframe=black!70,
  coltitle=white,
  title=\textbf{#2},
  fonttitle=\bfseries,
  boxrule=0.5mm,
  arc=2mm,
  enhanced,
  attach boxed title to top left={xshift=5mm,yshift=-3mm},
  boxed title style={colback=black!70},
  #1
}

\definecolor{pdecolor}{RGB}{0, 100, 180}
\definecolor{codegencolor}{RGB}{0, 130, 60}
\definecolor{refinecolor}{RGB}{180, 60, 0}
\definecolor{evalcolor}{RGB}{120, 0, 150}

\theoremstyle{definition}
\newtheorem{definition}{Definition}[section]
\theoremstyle{plain}
\newtheorem{proposition}{Proposition}[section]

\newtheorem{corollary}{Corollary}[section]
\theoremstyle{remark}
\newtheorem{remark}{Remark}[section]

\definecolor{cvprblue}{rgb}{0.21,0.49,0.74}
\usepackage[pagebackref,breaklinks,colorlinks,allcolors=cvprblue]{hyperref}

\AtEndPreamble{
\crefname{definition}{definition}{definitions}
\Crefname{definition}{Definition}{Definitions}
\crefname{proposition}{proposition}{propositions}
\Crefname{proposition}{Proposition}{Propositions}
\crefname{theorem}{theorem}{theorems}
\Crefname{theorem}{Theorem}{Theorems}
\crefname{corollary}{corollary}{corollaries}
\Crefname{corollary}{Corollary}{Corollaries}
\crefname{remark}{remark}{remarks}
\Crefname{remark}{Remark}{Remarks}
}

%% file: sec/0_abstract.tex
\begin{abstract}
Code generated by large language models (LLMs) is commonly accepted once it runs. We show that, for scientific simulation, this acceptance is unsafe. Running the single-shot outputs of three LLMs on MooseBench, a 220-case multiphysics benchmark that we release, we find that 41--60\% of the MOOSE inputs that run, and 39--56\% of those that run to completion, violate the expert reference: they run while solving a different problem. The gap persists in newer models such as GPT-6 Sol (68\% of its runnable inputs), even when the model first writes down the specification it intends to implement. These violations usually change the solution: on a 98-case subset, 56 of 63 comparable violating outputs deviate from the reference solution by more than 1\%, against 6 of 48 outputs that satisfy every check. Such failures are invisible to execution but checkable from the code. In MOOSE, each Kernel and boundary-condition object adds a known weak-form term, so our verifier reads the encoded PDE off the input and compares it with the intended one, with no solver run and no LLM judge. It accepts all 962 physics-preserving rewrites of the expert references and flags 1,286 of 1,313 physics-changing mutations without accepting any. As repair feedback, its report helps as much as the specification is right: checked against a model's own reading of the request, it can carry that reading's errors into the code; as they are corrected, it turns from worse than a self-check to better. We hope these results encourage verifying simulation code against the equations it is meant to solve, not only against whether it runs.
\end{abstract}

%% file: sec/1_intro.tex
\section{Introduction}
\label{sec:introduction}
Code generated by large language models (LLMs) is usually judged by running it: if it executes and passes its tests, it is taken to be correct~\citep{chen2021evaluating}. For scientific simulation, this proxy is weaker than it looks. The specification of a simulation is a mathematical model---governing equations, boundary and initial conditions, coefficients and a temporal regime---and an input file can run, mesh and converge while encoding a different model. A converged solve certifies that a system of equations was solved, not that it was the intended system. Multiphysics simulation is central to engineering design and scientific discovery, especially when experiments are prohibitive~\citep{keyes2013multiphysics}; in safety-critical settings such as nuclear reactor operation, predicted core behavior depends on which physics the model couples~\citep{lee2017multi}. As AI agents increasingly write these simulations, checking which problem their code encodes becomes part of making AI for science trustworthy. \emph{A simulation that runs may still solve the wrong problem.}

We study this failure in MOOSE~\citep{permann2020moose}, a widely used open-source finite-element platform in which a simulation is specified by an input file of meshes, variables, kernels, boundary conditions, materials and solvers. \Cref{fig:simulation_error} shows five inputs, each with one change to the requested specification: each converges to a plausible field and solves a different problem (details in Appendix~\ref{app:simulation_error}). In LLM outputs, such failures are not rare. Running the single-shot outputs of three LLMs on 220 requests, we find that 41--60\% of the inputs that run, and 39--56\% of those that run to completion, carry at least one confirmed violation of the expert reference (\Cref{tab:runnable_wrong}). Only 14--27\% of all cases both run and satisfy every check. Nearly all violations lie in the physics itself, and they are not cosmetic: on a 98-case subset, 56 of 63 comparable violating outputs deviate from the reference solution by more than 1\%, against 6 of 48 outputs that satisfy every check (\Cref{fig:direct_numerics}).

\begin{figure*}[t]
    \centering
    \includegraphics[width=1\linewidth]{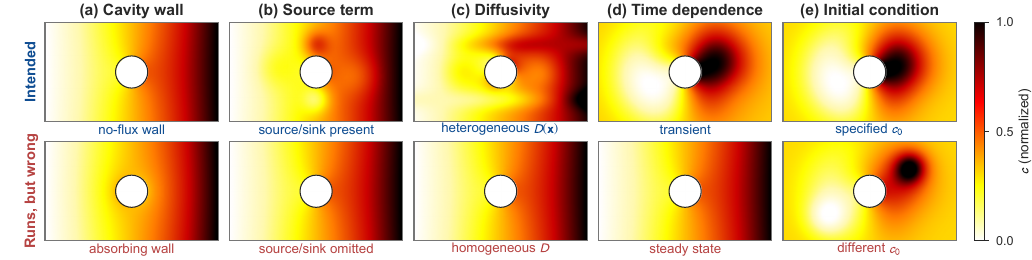}
    \caption{
    \textbf{Five runs that converge, and five wrong problems.}
    Each column is one cavity-diffusion problem. Top: the intended solution. Bottom: the solution of an input that runs to completion and produces a plausible field, but changes the specification in the one place named above the column. An execution check accepts all five.
    }
    \label{fig:simulation_error}
\end{figure*}

Recent LLM systems generate simulation inputs from natural language, including MOOSEnger for MOOSE, MetaOpenFOAM and Foam-Agent for CFD and AutoFLUKA for radiation transport~\citep{li2026moosenger,chen2024metaopenfoam,yue2025foamagent,ndum2024autofluka}. They check syntax, structure and execution, and MOOSEnger additionally scores manufactured-solution problems. None of these signals reconstructs the equations that a generated input actually encodes. The gap has been seen before: LLM agents that drive COMSOL produce executable API calls 88\% of the time yet rarely reach the correct target value~\citep{mudur2025feabench}, and generated infrastructure code mostly deploys after iterative repair but covers only 25.2\% of the user's requirements~\citep{zhang2025deployability}.

MOOSE makes the encoded physics readable. It assembles the weak-form residual as a sum of known contributions, one from each Kernel and boundary-condition object (\cref{sec:background}), so the active objects, with the materials that supply their coefficients, determine the problem that an input encodes. This makes verification simple: the verifier reads the encoded problem off the input and compares it, item by item, with the intended one. It runs no solver and asks no LLM for a judgment. Each check is satisfied, violated or unresolved; what the verifier cannot reconstruct is reported, never accepted. We call agreement between the two problems \emph{intent fidelity} and summarize it by the Intent Fidelity Score (IFS, \cref{sec:method:ifs}).

Because every claim we make passes through the verifier, we test it as one would test an instrument. Physics-preserving rewrites of the 220 expert references (renaming, reordering, regrouping, splitting a linear term) should leave the verdict unchanged; physics-changing mutations should be flagged at the changed item. All 962 applicable rewrites leave the verdict unchanged, whereas a direct comparison of objects and parameters flags 742 of them; 1,286 of 1,313 mutations are flagged, and none of the remaining 27 is accepted (\Cref{tab:verifier_audit}). Solving a subset in MOOSE confirms that the rewrites leave the solution unchanged (77 of 77) while the mutations change it (75 of 80 comparable mutations).

Finally, we ask whether the verifier's report helps a model repair its own input. Starting from the same initial input, we compare no feedback, a self-check, the runtime log and the report across seven models (\cref{sec:exp:matched}). Against the model's own reading of the request, the report helps some models and hurts others: it faithfully reports disagreements with a specification that is itself wrong, and the model edits correct code to match. Correcting that specification step by step isolates the cause: the report's advantage over the self-check grows at every step and turns from negative to positive for all three models we test, and with the expert specification the report beats the self-check for all four models and removes violations without making fewer inputs run. Verification is an effective repair signal; what limits it is the specification it verifies against.

In total, we audit 2,200 generated inputs from ten model settings and 9,380 repair outcomes, apply 2,275 transformations to the expert references, and solve 467 generated or transformed inputs in MOOSE. In summary, our contributions are:
\begin{enumerate}
    \item \textbf{A measurement.} On MooseBench, a 220-case benchmark with audited expert references, LLM-generated MOOSE inputs often run, even to completion, while solving a different problem, across ten model settings including the GPT-6 series.
    \item \textbf{An instrument and its audit.} A deterministic reconstruction of MOOSE inputs into PDE contracts with three-state checks, validated by rewrites, mutations, MOOSE replay and the solved fields of real LLM outputs.
    \item \textbf{Verification as feedback.} A same-start comparison showing that the report helps repair as much as the specification it checks is right, so that extracting the specification is what limits it in deployment.
\end{enumerate}

%% file: sec/2_related.tex
\section{Related Work}
\label{sec:related}

\paragraph{LLM systems for simulation code.}
MechAgents~\citep{ni2024mechagents} uses collaborating LLM agents for mechanics, MetaOpenFOAM~\citep{chen2024metaopenfoam}, OpenFOAMGPT~\citep{pandey2025openfoamgpt} and Foam-Agent~\citep{yue2025foamagent} target CFD, and AutoFLUKA~\citep{ndum2024autofluka} targets Monte Carlo radiation transport. MOOSEnger~\citep{li2026moosenger} wraps a model in a MOOSE-specific generate--check--repair--run harness; on manufactured-solution cases, an LLM judge decides whether an input is semantically aligned with the request, and a numerical criterion checks its accuracy. \mbox{FEABench}~\citep{mudur2025feabench} scores LLM agents that drive COMSOL by executability, by the error of a target value, and by how much of the reference model they recall: its interfaces, features, feature properties and model tree. AutoMOOSE~\citep{manna2026automoose} tests finished phase-field runs against conservation laws, asymptotic limits and scaling relations. These signals compare API actions, object names, a judge's verdict, end values or properties of the solution. Lang-PINN~\citep{he2025langpinn} parses generated PINN loss code back into a symbolic PDE and keeps it only if it matches the PDE its own agent formulated. Our verifier instead reconstructs the problem a MOOSE input encodes: each governing term with its coefficient, coupling and blocks, and each condition with its type and value. It normalizes this problem so that physics-preserving rewrites leave it unchanged, compares it item by item with the requested one, and reports what it cannot reconstruct as unresolved rather than as a match.

\paragraph{Verifying generated code against user intent.}
Intent verification for generated code has been studied through correct-by-construction translation of formal specifications~\citep{schellhorn2024vericode}, LLM-generated postconditions that formalize natural-language intent~\citep{endres2024}, and policy-guided verifier feedback for Infrastructure-as-Code~\citep{terraformer}. These methods must reason about what a program does, through proofs, generated assertions or learned verifiers. In MOOSE we can instead read what an input encodes (\cref{sec:background}): the encoded problem is reconstructed deterministically, and verification reduces to a structural comparison between two physics contracts. Our verifier also sits upstream of traditional simulation verification and validation (V\&V)~\citep{oberkampf2010verification}, which asks whether the equations are solved right and whether they are the right equations, and in both cases presupposes that the input file states the equations its author intended.

\paragraph{Self-refinement and PDE representations.}
LLM self-refinement methods iteratively improve generated code via the model's own critique~\citep{madaan2023self,reflexion} or runtime traces~\citep{self_debugging}. Critique comes from the same model that generated the code, and runtime traces report only what fails to run, so an error that neither the model nor the solver notices stays unseen. Our report instead lists the deterministic differences between the encoded and the requested problem, so a missing or extra term is named whether or not the model would notice it. Prior machine-readable PDE representations serve models that solve equations symbolically~\citep{Lample2020Deep} or learn their solution operators~\citep{DBLP:journals/corr/abs-2010-08895}, form languages that generate solver code~\citep{alnaes2012ufl}, and physics-informed learning~\citep{raissi2019physics}. Ours serves verification, which needs only a deterministic comparison of two discrete structures and therefore a coarser descriptor.

%% file: sec/3_moose.tex
\section{Background: Kernels as Weak-Form Terms}
\label{sec:background}
\label{sec:background:moose}

The MOOSE framework~\citep{permann2020moose} solves systems of partial differential equations with the finite element method in residual form: it seeks the discrete solution $u$ whose weak-form residual vanishes for every test function $\psi$,
\begin{equation}
\mathcal{R}(u;\psi)=\sum_{k} r_k(u;\psi)+\sum_{b} r_b(u;\psi)=0 .
\label{eq:residual_sum}
\end{equation}
Each \emph{Kernel} $k$ contributes a volume integral $r_k$ and each integrated boundary condition $b$ a boundary integral $r_b$; Dirichlet conditions constrain $u$ directly. An input file lists exactly these objects, so it states the weak form term by term. For example, the weak form of the zero-gravity Darcy pressure equation is
\begin{equation}
\underbrace{\left(\nabla \psi,\frac{K}{\mu}\nabla p\right)}_{\texttt{Kernel}}
-
\underbrace{\left\langle \psi,\frac{K}{\mu}\nabla p \cdot \hat{n}\right\rangle}_{\texttt{Boundary Condition}}
= 0,
\label{eq:darcy_weak}
\end{equation}
where $(\cdot,\cdot)$ and $\langle\cdot,\cdot\rangle$ denote integrals over the domain and its boundary. Where no boundary object is given, the boundary term is absent and the natural no-flux condition holds. The derivation of \eqref{eq:darcy_weak} and the corresponding input file are given in Appendix~\ref{app:darcy}. For the Kernels covered in \Cref{tab:full_kernel}, enumerating the active Kernel and boundary-condition objects of an input therefore reconstructs its encoded residual. \Cref{fig:kernel_pde} illustrates this mapping on a coupled thermomechanical example.

\begin{figure}[t]
    \centering
    \includegraphics[width=\linewidth]{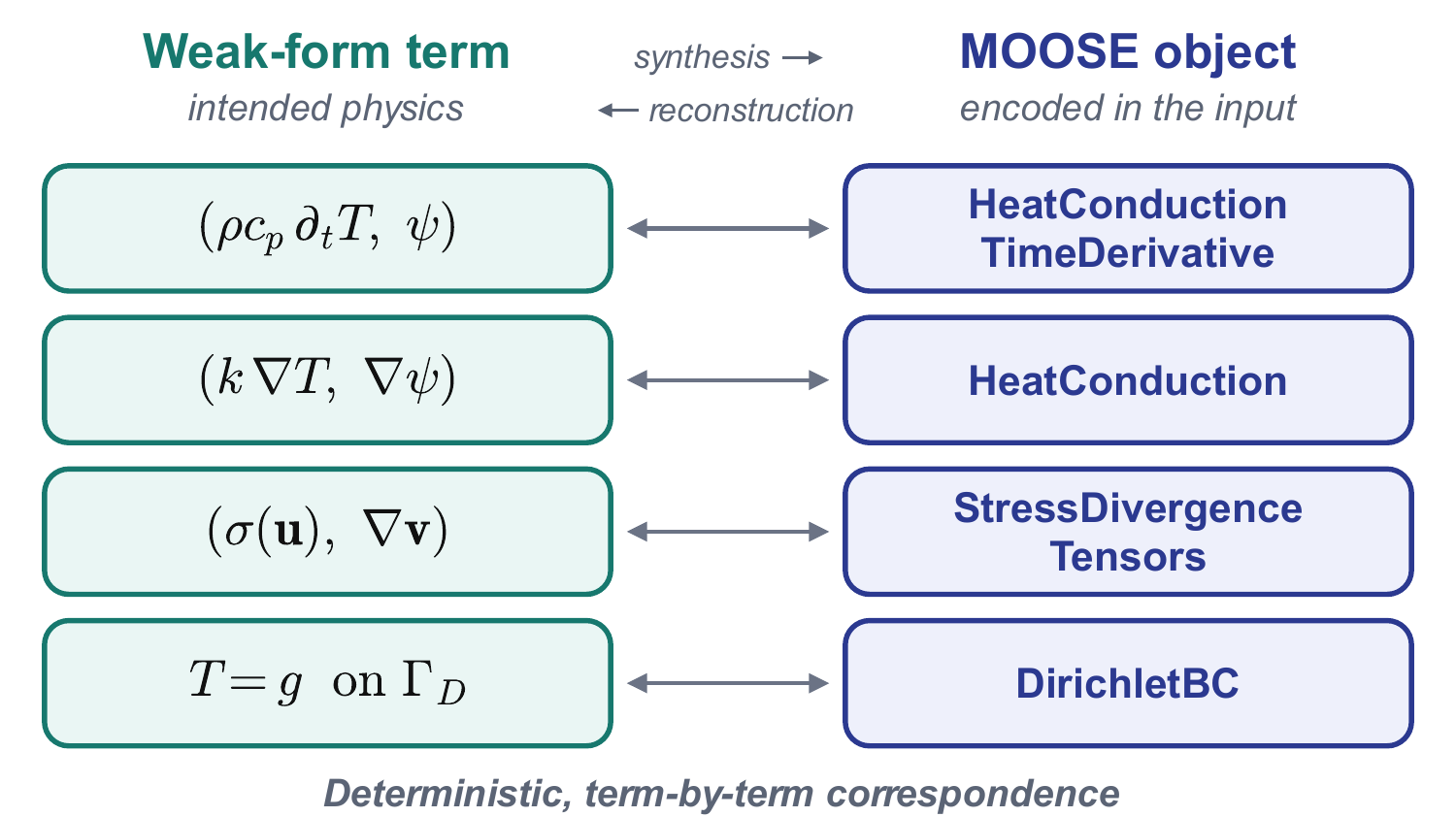}
    \caption{Weak-form terms map to MOOSE Kernel/BC objects.}
    \label{fig:kernel_pde}
\end{figure}

\paragraph{Valid objects are not the intended physics.}
An object is valid if it satisfies its schema, such as required parameters and valid types; this says only that the object can exist in the input file. Intent fidelity asks whether the chosen objects together encode the weak form that the user intended.

Term comparison is keyed by the variable and a normalized PDE operator type from the kernel--PDE mapping table, which also records, for each operator type, the weak-form contribution that justifies it (Appendix~\ref{app:formalization}).

%% file: sec/4_method.tex
\section{Verifying the Encoded Problem}
\label{sec:method}

\subsection{Problem Formulation}
\label{sec:method:formulation}
Given a natural-language description $d$ of a simulation, a model generates a MOOSE input $c$. We ask whether $c$ encodes the problem that $d$ requests. Both sides are written as \emph{physics contracts}, our formal version of a simulation specification: $\mathcal{P}_{\mathrm{code}}(c)$ is reconstructed deterministically from the code, and a reference contract $\mathcal{P}_{\mathrm{ref}}$ states what is required. For benchmark evaluation the reference is the expert contract $\mathcal{P}_{\mathrm{gt}}$; for repair it is the contract $\mathcal{P}_{\mathrm{llm}}$ that the model extracts from $d$, so no expert information reaches generation. The verifier compares $\mathcal{P}_{\mathrm{ref}}$ with a candidate $\mathcal{P}_{\mathrm{cand}}=\mathcal{P}_{\mathrm{code}}(c)$ and returns a set of checks, each \emph{satisfied}, \emph{violated} or \emph{unresolved}.

\subsection{Physics Contracts}
\label{sec:method:pde}
A contract is the tuple
\begin{equation}
\label{eq:physics_repr}
\mathcal{P} = \left(\mathcal{T},\; \mathcal{B},\; \mathcal{I},\; \mathcal{S},\; \mathcal{C},\; \mathcal{U}\right),
\end{equation}
where $\mathcal{T}$ holds the contributions to the governing equations, $\mathcal{B}$ the boundary conditions, $\mathcal{I}$ the initial conditions, $\mathcal{S}$ the temporal regime (steady or transient), $\mathcal{C}$ the spatial and temporal context (domain extent, dimension, start and end time), and $\mathcal{U}$ the items that could not be recovered. Each contribution is a bound instance
\begin{equation}
\label{eq:term}
t = \left(u,\; o,\; \alpha,\; \mathbf{v},\; R,\; \mathbf{m}\right)
\end{equation}
with target variable $u$, normalized operator type $o$, coefficient $\alpha$, coupled variables $\mathbf{v}$, the blocks $R$ on which it acts, and constitutive facts $\mathbf{m}$ supplied by materials (\eg, elastic or plastic parameters). Coefficients and material facts are thus checked together with the term they act on. A boundary condition records its variable, boundary set, type and value, plus the transfer coefficient and ambient value of a Robin condition; an initial condition records its variable, type and value.

Reconstruction is deterministic for the covered MOOSE fragment. A table $\mathcal{M}$ maps each Kernel class to a normalized operator type, a coefficient rule, a severity weight and, where available, the weak-form contribution that justifies the assignment (\cref{sec:background:moose}); boundary and initial conditions are mapped likewise, and materials and functions are resolved to the coefficients they supply. Objects outside $\mathcal{M}$, and values that cannot be recovered, such as a coefficient set by a tabulated function, enter $\mathcal{U}$ with a reason instead of being guessed. The table documents a source-grounded weak-form contribution for 762 of the 766 kernel-term instances in the MooseBench references (Appendices~\ref{app:kernel_table} and~\ref{app:formalization}); the reconstruction algorithm is in Appendix~\ref{app:architecture}.

\begin{figure*}[t]
\centering
\includegraphics[width=\linewidth]{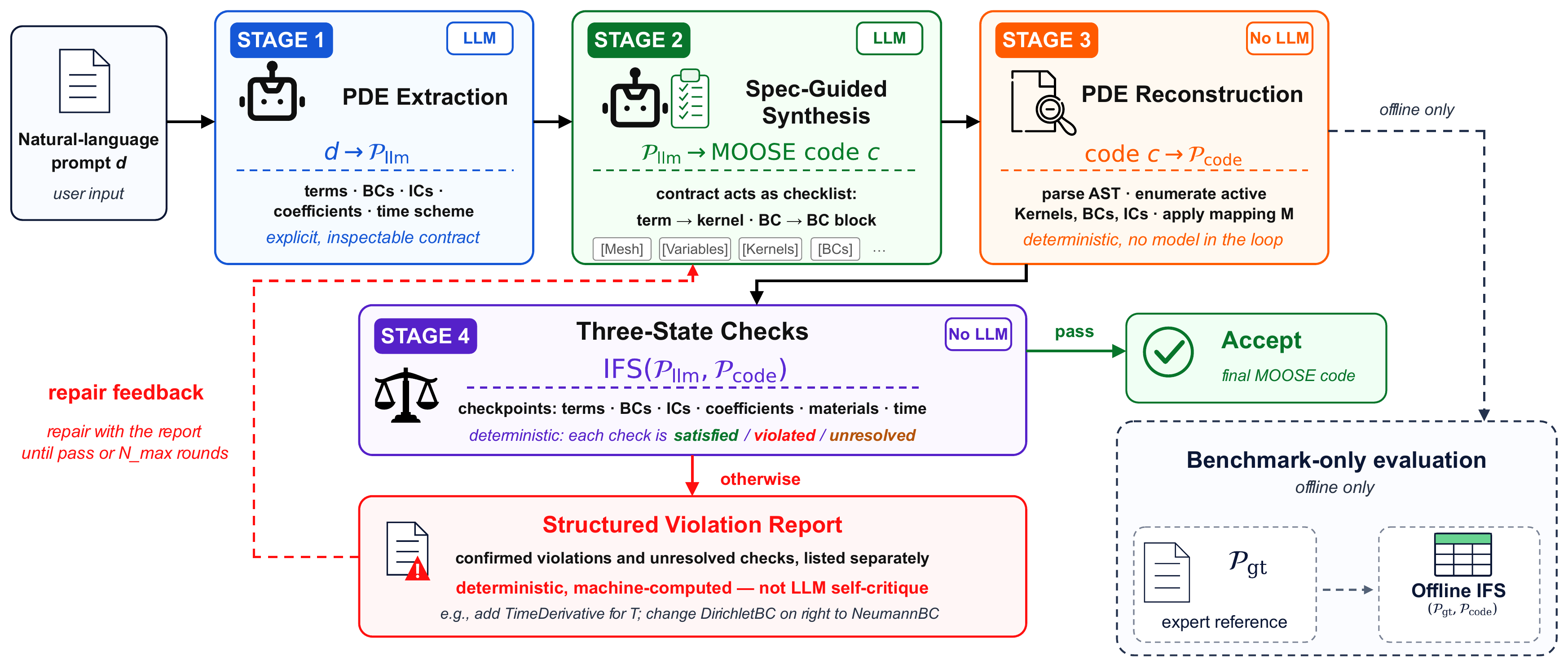}
\caption{\textbf{The verifier and its two uses.} The verifier (Stages 3--4) reconstructs $\mathcal{P}_{\mathrm{code}}$ from a MOOSE input and judges every check satisfied, violated or unresolved, with no model in the loop. For evaluation, it compares the input with the expert contract $\mathcal{P}_{\mathrm{gt}}$, which the deployed pipeline never sees. As repair feedback, it compares the input with the contract $\mathcal{P}_{\mathrm{llm}}$ extracted in Stage 1 and reports confirmed violations and unresolved checks separately. In deployment the loop can run for several rounds; our experiments evaluate a single repair round from a shared start, which isolates the feedback (\cref{sec:exp:matched}).}
\label{fig:system}
\end{figure*}

\subsection{Three-State Checks}
\label{sec:method:ifs}
\paragraph{Normalization.}
Both contracts are first put in a canonical form that leaves the encoded problem unchanged. Because the residual is a sum (\cref{eq:residual_sum}), constant linear contributions with identical bindings are summed and exact zeros dropped; zero-flux Neumann conditions are removed where they are the natural condition of a scalar diffusion--reaction equation; and a zero constant initial condition is removed because MOOSE applies it by default. When both contracts are steady, an initial value on a solved variable is only the starting guess of the nonlinear solver, not part of the problem; it is recorded but not checked. Initial values of auxiliary fields, and all initial conditions of transient problems, are checked.

\paragraph{Alignment and matching.}
Candidate variables are aligned to reference variables by name and then by their multisets of operator types; where the alignment stays ambiguous, the affected checks are unresolved rather than guessed. Contributions are grouped by variable, operator type and blocks, boundary conditions by variable and boundary set, and initial conditions by variable. Within each group, reference and candidate instances are paired one to one by a minimum-violation assignment.

\paragraph{Checks.}
A paired contribution is checked for presence, coefficient, coupling and constitutive facts; a paired condition for type, value and Robin parameters; the temporal regime gives one more check. Numerical values agree when their relative difference is at most $10^{-6}$. A reference instance without a partner is a violated \emph{missing} check, and a candidate instance without a partner a violated \emph{extra} check, unless the variable alignment is ambiguous or an unmapped object on either side could account for the difference; the check is then unresolved. A check is also unresolved when a value on either side was not recovered, or when two symbolic expressions are not textually identical, since their equivalence is not decided. Every unresolved check records its origin: the reference, the candidate, the alignment or a symbolic comparison.

\paragraph{Scores and verdicts.}
Let $\mathcal{Q}=\mathcal{Q}(\mathcal{P}_{\mathrm{ref}},\mathcal{P}_{\mathrm{cand}})$ be the resulting checks, at least one per reference item and one per unpaired candidate item, with weights $w_j>0$ from the severity table. The satisfied and violated shares
\begin{equation}
\label{eq:ifs}
S = \frac{\sum_{j \in \mathcal{Q}_{\mathrm{sat}}} w_j}{\sum_{j \in \mathcal{Q}} w_j},
\qquad
V = \frac{\sum_{j \in \mathcal{Q}_{\mathrm{vio}}} w_j}{\sum_{j \in \mathcal{Q}} w_j}
\end{equation}
leave $1-S-V$ for the unresolved share. We report $S$ as the Intent Fidelity Score (IFS); it is not $1-V$, because an unresolved check counts as neither satisfied nor violated. Checks of the context $\mathcal{C}$ are reported separately and do not enter $S$. An input has a \emph{confirmed violation} if any check, including a context check, is violated. Otherwise it is \emph{unresolved-only} if any check is unresolved or either contract lists an unrecovered item ($\mathcal{U}\neq\emptyset$), and \emph{satisfied} only if every check is satisfied and nothing was left unrecovered; an input with an object outside $\mathcal{M}$ is therefore never reported as satisfied. These are the categories of \Cref{tab:runnable_wrong}. They are defined by the outcomes of the checks; the weights affect only $S$ and the pairing.

By construction (Appendix~\ref{app:formalization}), $S=1$ exactly when every scored check is satisfied: the candidate then meets every requirement that the reference represents outside the context and adds no unpaired item (\Cref{thm:ifs_soundness}), while a missed requirement or an unpaired item leaves a check violated or unresolved, so that $S<1$ (\Cref{thm:representational_completeness}). A \emph{satisfied} verdict requires, in addition, that the context checks pass and nothing is left unrecovered. The checks are also invariant under the rewrites that the canonical form absorbs, such as reordering objects, splitting a constant linear term or regrouping boundary lists (\Cref{prop:invariance}); \cref{sec:exp:verifier_audit} tests both properties empirically.

\subsection{The Report as Feedback}
\label{sec:method:refinement}
The checks are rendered as a two-part report: confirmed violations, each with the expected and the found value, and checks that could not be verified, each with its reason (\Cref{fig:system}). Keeping the two lists apart prevents an unresolved item from being presented to the model as an error. During repair the reference is $\mathcal{P}_{\mathrm{llm}}$, so the report checks the code against the model's own reading of the request. In deployment the report can drive several repair rounds; in the same-start comparison (\cref{sec:exp:matched}) each arm makes a single repair call, so that the arms differ only in the feedback they receive.

%% file: sec/5_benchmark.tex
\section{MooseBench: A Multiphysics Benchmark with Expert References}
\label{sec:benchmark}

We construct \textbf{MooseBench}, a benchmark of 220 multiphysics simulation requests, each with an expert reference input and the PDE contract reconstructed from it. The benchmark spans seven physics families and three complexity tiers; the full family distribution and tier definitions are reported in Appendix~\ref{app:moosebench} (Table~\ref{tab:stratified}). Prompts vary in verbosity and engineering context and describe the target PDE, boundary and initial conditions, coefficients, geometry and time interval in domain language, without MOOSE object names.

\subsection{Ground Truth Construction and Audit}
\label{sec:gt-audit}
For each prompt, an expert-written reference MOOSE input encodes the intended governing equations, boundary and initial conditions, material properties and temporal regime. Applying the deterministic reconstruction in \Cref{sec:method:pde} gives $\mathcal{P}_{\mathrm{gt}}$. Because $\mathcal{P}_{\mathrm{gt}}$ is produced by the same reconstruction that the verifier uses, the audit reviews it jointly with the prompt and the reference input, checking prompt faithfulness, operator coverage, consistency of boundary and initial conditions, materials and temporal regime, and prompt leakage. Kernel variants are canonicalized to normalized PDE operator types by the mapping table, and all benchmark artifacts are released in the supplementary material. One revision corrected 38 reaction terms in 34 references; all verdicts in \cref{sec:experiments} use the revised set.

\FloatBarrier

%% file: sec/6_experiments.tex
\section{Experiments}
\label{sec:experiments}

\subsection{Experimental Setup}
\label{sec:exp:setup}

\paragraph{Models.}
\Cref{sec:exp:runnable_wrong,sec:exp:direct_numerics} evaluate the \emph{Direct} outputs, generated from the request in a single call without feedback, of Claude Sonnet 4.6~\citep{anthropic2026sonnet46}, GPT-5.4~\citep{openai2026gpt54} and DeepSeek V4 Flash~\citep{deepseekai2026deepseekv4}. The same-start comparison (\cref{sec:exp:matched}) uses seven models: Qwen3-14B and Qwen3-32B~\citep{qwen2025qwen3} in non-thinking mode, served locally with fixed weights; GPT-5.4 and GPT-5.4 mini~\citep{openai2026gpt54,openai2026gpt54mini} through the OpenAI API with dated snapshots; and the newer GPT-6 Luna, GPT-6 Sol~\citep{openai2026gpt6sol} and GPT-6.1 Terra~\citep{openai2026gpt61terra}, whose API names carry no date. API models run with reasoning effort none, the setting that also produced the Direct outputs of GPT-5.4. All models decode greedily (temperature $0$), one sample per request. Weight revisions, snapshot identifiers and call dates are in Appendix~\ref{app:repro}.

\paragraph{Execution check.}
An output is \emph{runnable} if it passes our execution check: MOOSE exits with code 0, starts a solve before the first error, or reports no error within 2 seconds. Most runnable Direct outputs (272 of 348) also run to completion.

\paragraph{References.}
All verdicts are computed against the revised expert references (\cref{sec:benchmark}); generation and repair never see them, except in the two controls of \cref{sec:exp:matched} that correct the extracted contract with them. Same-start outputs pass through a frozen deterministic layer, identical in every arm, that maps object types and parameter names onto MOOSE's registry and adds solver tolerances; it reads neither contract nor reference (Appendix~\ref{app:registry-note}).

\subsection{Does the Verifier See Physics or Syntax?}
\label{sec:exp:verifier_audit}
If the verifier reacted to surface form, rewrites that keep the physics would be flagged; if it missed physics, mutations that change it would pass. Neither happens. We transform the 220 expert references. Rewrites that preserve the physics (renaming variables or materials, swapping a constant for an equivalent material property, splitting a linear term, reordering blocks, regrouping boundary lists) should leave every check satisfied. Mutations (removing or duplicating a term, scaling a coefficient, changing a coupling, a boundary type or value, the time regime, an initial condition or a reaction sign) should be flagged as violations at the changed item.

Of 23 transformation types, 21 apply to at least one reference (\Cref{tab:verifier_audit}); merging linear terms and moving a material to another block apply to none. All 962 applicable rewrites leave the verdict unchanged; a direct comparison of object types and parameter values, the natural alternative, flags 742 of them, every rewrite except a reordering. Of 1,313 mutations, 1,286 are flagged as violations in the changed category and variable, 9 material-coefficient changes end as unresolved, and 18 term deletions leave no reconstructable equation; none is accepted. Re-running 173 transformed inputs from the numerical subset of \cref{sec:exp:direct_numerics} in MOOSE leaves all 77 rewritten fields unchanged; 75 of the 80 comparable mutations change the field and 5 leave the sampled field unchanged; 16 are not comparable, mostly because the time regime changed (per type in Appendix~\ref{app:verifier_audit}).

\begin{table}[t]
\centering
\caption{\textbf{Verifier audit} on the 220 expert references. Rewrites should leave the verdict unchanged; mutations should be flagged in the changed category. Numerical replay re-runs a subset in MOOSE.}
\label{tab:verifier_audit}
\footnotesize\setlength{\tabcolsep}{2.5pt}
\begin{tabular}{@{}lccccc@{}}
\toprule
& \textbf{Applic.} & \textbf{As expected} & \textbf{Unres.} & \textbf{No recon.} & \textbf{MOOSE replay} \\
\midrule
Rewrites & 962 & 962 & 0 & 0 & 77/77 unchanged \\
Mutations & 1{,}313 & 1{,}286 & 9 & 18 & 75/80 changed \\
\bottomrule
\end{tabular}
\end{table}

\subsection{How Often Does Wrong Code Run?}
\label{sec:exp:runnable_wrong}
\paragraph{Many inputs that run, even to completion, encode a different problem.}
We pass the Direct outputs of the three models (220 each) through the execution check and the verifier. Among runnable outputs, \Cref{tab:runnable_wrong} separates those with at least one confirmed violation, those whose only open items are unresolved checks, and those whose represented checks are all satisfied. Between 41\% (Sonnet 4.6) and 60\% (DeepSeek V4 Flash) of runnable outputs carry a confirmed violation, \ie, 25--30\% of all cases. An execution check accepts all of them; only 14--27\% of all cases both run and satisfy every represented check. The picture is the same for the outputs that run to completion: 39--56\% of them violate the reference. The violations are in the physics: 173 of the 180 runnable violating outputs violate a term, coefficient, material, boundary or initial condition (coefficients in 104, boundary conditions in 83; \Cref{tab:violation_categories}), and only 7 differ only in domain or time horizon.

\paragraph{Newer models and an explicit contract do not remove the gap.}
The lower block of \Cref{tab:runnable_wrong} audits the initial inputs of \cref{sec:exp:matched}, which each model writes together with the contract it extracted from the request. Even so, 57--68\% of the runnable inputs of GPT-5.4, GPT-5.4 mini, GPT-6 Luna and GPT-6 Sol carry a confirmed violation (GPT-6 Sol: 97 of 143, or 44\% of all cases); GPT-6.1 Terra is lower at 30\%. The Qwen3 models rarely produce a runnable input (9 and 8 of 220), consistent with the 0\% that MOOSEnger reports for an unassisted open 31B model~\citep{li2026moosenger}.

\begin{table}[t]
\centering
\caption{\textbf{Runnable but wrong.} 220 cases per row. Columns 3--5 partition the outputs that pass the execution check (Runnable) into those with a confirmed violation (share of runnable), only unresolved checks, or all represented checks satisfied. Upper block: Direct outputs from the request alone. Lower block: initial inputs of the same-start comparison, generated together with the model's own extracted contract. $^\dagger$Includes runnable inputs that could not be reconstructed and so fall in none of columns 3--5: all 9 for Qwen3-14B and 1 for Qwen3-32B.}
\label{tab:runnable_wrong}
\footnotesize\setlength{\tabcolsep}{3pt}
\begin{tabular}{@{}lcccc@{}}
\toprule
\textbf{LLM} & \textbf{Runnable} & \textbf{Violation} & \textbf{Unres.\ only} & \textbf{Satisfied} \\
\midrule
\multicolumn{5}{@{}l}{\emph{Direct, request only}} \\
Claude Sonnet 4.6 & 145 & 59 (41\%) & 27 & 59 \\
GPT-5.4 & 112 & 66 (59\%) & 13 & 33 \\
DeepSeek V4 Flash & 91 & 55 (60\%) & 6 & 30 \\
\midrule
\multicolumn{5}{@{}l}{\emph{Initial input, with extracted contract}} \\
GPT-5.4 & 132 & 75 (57\%) & 18 & 39 \\
GPT-5.4 mini & 78 & 52 (67\%) & 9 & 17 \\
GPT-6 Luna & 110 & 73 (66\%) & 8 & 29 \\
GPT-6 Sol & 143 & 97 (68\%) & 9 & 37 \\
GPT-6.1 Terra & 133 & 40 (30\%) & 18 & 75 \\
Qwen3-14B & 9$^\dagger$ & 0 & 0 & 0 \\
Qwen3-32B & 8$^\dagger$ & 5 & 1 & 1 \\
\bottomrule
\end{tabular}
\end{table}

\subsection{Do Violations Change the Solution?}
\label{sec:exp:direct_numerics}
A confirmed violation says that the encoded problem differs from the reference; it does not by itself say that the solution differs. We therefore solve the Direct outputs of a 98-case subset, drawn by family and complexity tier in a fixed hash order before any generation, in MOOSE and compare their fields with the reference solution at sample points fixed in advance. For each field variable we compute the relative $\ell_2$ error over the sample points at the final time, $\lVert\mathbf{u}-\mathbf{u}_{\mathrm{ref}}\rVert_2/\lVert\mathbf{u}_{\mathrm{ref}}\rVert_2$, and for each case we take the largest over its variables. We compare outputs with a confirmed violation against outputs whose represented checks are all satisfied; the latter measure the error that remains from differences in mesh and time step. Outputs that fail, time out or use a different domain or time horizon are counted separately rather than dropped.

\begin{figure}[t]
\centering
\includegraphics[width=\linewidth]{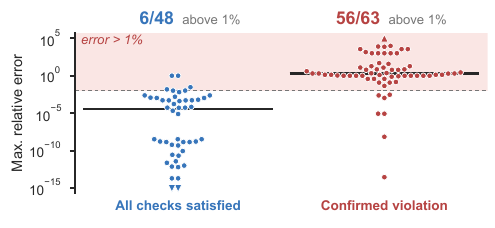}
\caption{\textbf{Do violations change the solution?} Per-case maximum relative $\ell_2$ error at fixed sample points of the completed, comparable Direct outputs on the 98-case subset, by verdict; lines mark medians, shading errors above 1\%. Violating outputs (63; median 1.93) are off by about their own magnitude; satisfied outputs (48; median $3.7{\times}10^{-5}$) agree closely. Errors below $10^{-15}$ and above $10^{5}$ are drawn as triangles at the axis limits. Per-model statistics are in \Cref{tab:direct_numerics_detail}.}
\label{fig:direct_numerics}
\end{figure}

\paragraph{Violations change the solution; satisfied outputs mostly reproduce it.}
In every model the violating outputs have the larger errors ($p\le1.1{\times}10^{-5}$), and their median error is more than four orders of magnitude above that of the satisfied outputs. A field that vanishes at the sample points has error 1 and one of reversed sign has error 2, so the violating outputs are off by their own magnitude, not by a discretization difference. Read as a test, a confirmed violation flags an error above 1\% in 56 of 63 outputs, and a satisfied verdict clears 42 of 48; the few exceptions are explained in Appendix~\ref{app:direct_numerics}. Results with and without the steady initial-guess rule (\cref{sec:method:ifs}) agree.

\subsection{When Does the Report Help Repair?}
\label{sec:exp:matched}
All arms share one model configuration and one initial input, which the model writes after extracting a contract from the request. Each repair arm makes one repair call, kept whether or not it improves. Every repair prompt contains the request and the extracted contract; only the feedback differs: none (Initial), an instruction to check the input against them (Self), the runtime log (Runtime), or the verifier's report of that comparison (Report). Self and Report thus differ only in whether the model or the verifier makes the comparison.

\begin{table}[t]
\centering
\caption{\textbf{Same-start feedback comparison} with the model's own extracted contract, 220 cases per model. IFS of the initial input and the paired change of each repair arm against it; the last column gives Report$-$Self, with $^\dagger$ where its 95\% bootstrap CI excludes zero (CIs in \Cref{tab:weight_sensitivity}). Unreconstructable outputs count as IFS 0.}
\label{tab:matched}
\footnotesize\setlength{\tabcolsep}{2.2pt}
\begin{tabular}{@{}lccccc@{}}
\toprule
& \textbf{Initial} & \multicolumn{3}{c}{\textbf{$\Delta$ vs.\ Initial}} & \\
\cmidrule(lr){3-5}
\textbf{LLM} & \textbf{IFS} & \textbf{Self} & \textbf{Runtime} & \textbf{Report} & \textbf{Report$-$Self} \\
\midrule
GPT-5.4 & .811 & $+$.004 & $+$.008 & $-$.033 & $-$.037$^\dagger$ \\
GPT-5.4 mini & .620 & $+$.005 & $+$.017 & $-$.013 & $-$.018$^\dagger$ \\
GPT-6 Luna & .759 & $+$.004 & $+$.016 & $-$.045 & $-$.048$^\dagger$ \\
GPT-6 Sol & .798 & $+$.018 & $+$.007 & $+$.042 & $+$.024$^\dagger$ \\
GPT-6.1 Terra & .793 & $+$.021 & $-$.019 & $+$.004 & $-$.017 \\
Qwen3-14B & .193 & $+$.060 & $+$.038 & $+$.070 & $+$.011 \\
Qwen3-32B & .236 & $+$.052 & $+$.062 & $+$.076 & $+$.024$^\dagger$ \\
\bottomrule
\end{tabular}
\end{table}

\paragraph{With an extracted contract, the report helps some models and hurts others.}
Against the self-check, the report raises IFS for GPT-6 Sol and Qwen3-32B but significantly lowers it for GPT-5.4, GPT-5.4 mini and GPT-6 Luna (\Cref{tab:matched}). For the Qwen3 models every arm helps, mostly by making unreconstructable inputs reconstructable. The direction of every difference holds with equal check weights (Appendix~\ref{app:matched}).

\paragraph{The report turns errors of the extracted contract into edits.}
Where the report lowers IFS, the regressions share one pattern. For GPT-5.4 mini, 28 of the 35 cases in which the report output scores below the initial input follow one sequence: the report lists a confirmed violation against the extracted contract, the model removes it, and against the reference the repaired input has new violations, mostly in coefficients and terms. For GPT-5.4, 30 of 54 regressions show the same pattern. The self-check, given the same contract, produces this pattern in 0 and 2 cases. The model follows the report; the error lies in the contract the report checks against (Appendix~\ref{app:matched}). Execution would have favored these edits: for GPT-5.4, GPT-5.4 mini, GPT-6 Luna and GPT-6 Sol, more report than self-check outputs run (143 \vs 138 for GPT-5.4), though three have lower IFS.

\begin{figure}[t]
\centering
\includegraphics[width=\linewidth]{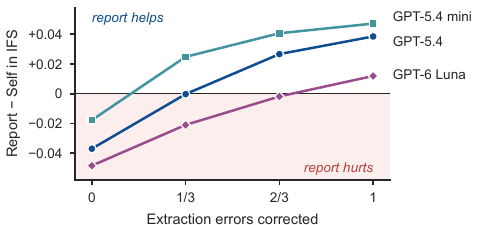}
\caption{\textbf{The report helps as much as the specification is right.} Report$-$Self in IFS from the same initial inputs (220 cases per model) as a growing share of the extracted contract's errors is corrected to the reference. Slopes are $+$0.076, $+$0.063 and $+$0.060 for GPT-5.4, GPT-5.4 mini and GPT-6 Luna, all with 95\% CIs above zero (\Cref{tab:spec_dose}).}
\label{fig:spec_dose}
\end{figure}

\paragraph{The report helps as much as the specification is right.}
To trace the three regressions of \Cref{tab:matched} to their cause, we correct the contract errors in steps: for each of these models, we list the entries of each extracted contract that disagree with the reference (a missing, extra or different term, condition or coefficient), order them at random, and correct the first third, two thirds or all of them; prompts, renderer and initial input stay fixed. Report$-$Self rises at every step for all three models (\Cref{fig:spec_dose}): it is negative for each with the extracted contract ($-$0.018 to $-$0.048) and positive for each once every correctable entry is fixed ($+$0.012 to $+$0.047). Both arms receive the same corrected contract, yet from the first step to the last the report's IFS rises by 0.077 to 0.104 and the self-check's by only 0.010 to 0.039: given the same information, the verifier turns a correct specification into correct code more reliably than the model's own check.

\paragraph{With the reference itself, the report helps every model.}
With the expert reference in place of the extracted contract, the report beats the self-check for the same four models, by $+$0.018 to $+$0.038. The gain comes from fixed violations: the violated share falls significantly in every model, the unresolved share does not change, and no fewer repaired inputs run (Appendix~\ref{app:matched}). The report is thus a repair signal as good as its specification; without a reference, the main open problem is extracting the specification, not checking against it.

%% file: sec/7_discussion.tex
\section{Discussion}
\label{sec:discussion}

\subsection{Beyond MOOSE}
\label{sec:exp:ufl}
The verifier needs only a map from language constructs to the terms of the PDE. UFL/FEniCS~\citep{alnaes2012ufl}, FreeFEM~\citep{freefem}, FiPy~\citep{guyer2009fipy} and Devito~\citep{lange2016devito} expose their equations term by term, so the same comparison could be built for them (Appendix~\ref{app:crossdsl}); languages that hide assembly behind solver calls would need a different target.

\subsection{Limitations}
\label{sec:discussion:limitations}
The verifier checks the structure of the encoded problem in the covered MOOSE fragment: terms, conditions, coefficients, materials and the temporal regime. It does not assess the coordinate system, mesh adequacy, solver convergence or discretization error, and objects outside its mapping are reported as unresolved. The comparison is literal: an equivalent reformulation it does not map, another unit system or a mirrored geometry is flagged, and violations are measured against the expert references (\cref{sec:benchmark}), so an error in a reference would count as a violation. In a nonlinear steady problem, the initial guess, which is not checked, can select among solutions. The feedback study uses one repair round and one greedy sample per arm; its graded and known-specification controls use the reference, so they show what a better specification is worth, not how to obtain one.
\FloatBarrier

%% file: sec/8_conclusion.tex
\section{Conclusion}
\label{sec:conclusion}

A simulation that runs may still solve the wrong problem. In MOOSE the encoded PDE can be read off the input and checked without running the solver. For the direct outputs of three LLMs, 41--60\% of the inputs that run, and 39--56\% of those that run to completion, encode a different problem; newer models show the same gap, and the difference usually changes the solution. As repair feedback, the check helps as much as the intended problem is stated correctly and propagates a model's misreading of the request otherwise, so the main open problem is extracting the specification, not checking against it. Simulation practice separates solving the equations right from solving the right equations~\citep{oberkampf2010verification}; when a model writes the input, the second must be checked, and where a problem is built from known terms, reading suffices.

%% file: sec/X_suppl.tex
\section{Details of the Silent-Failure Simulations}
\label{app:simulation_error}

Figure~\ref{fig:simulation_error} uses a diffusion problem on a perforated rectangular domain
\[
\Omega=\Omega_0\setminus\overline{\Omega}_{\rm cav},\qquad
\Omega_0=[-3.5,3.5]\times[-2,2],
\qquad
\Omega_{\rm cav}=\{(x,y):x^2+y^2<R^2\},\quad R=0.65 .
\]
The intended transient model is
\[
\partial_t c-\nabla\!\cdot(D(\mathbf{x})\nabla c)=s(\mathbf{x},t)
\quad \text{in } \Omega,
\qquad
\nabla c\cdot\mathbf{n}=0 \quad \text{on } \Gamma_{\rm cav},
\qquad
c(\mathbf{x},0)=c_0(\mathbf{x}).
\]
Steady-state cases are obtained by omitting the transient term \(\partial_t c\).

The five columns perturb different parts of this specification. Panel (a) replaces the no-flux cavity wall with an absorbing cavity boundary. Panel (b) omits the source/sink term. Panel (c) replaces the heterogeneous diffusivity \(D(\mathbf{x})\) with a homogeneous coefficient. Panel (d) solves the steady-state problem instead of the intended transient problem. Panel (e) uses an incorrect initial condition. All fields are normalized for visualization.

\section{Weak-Form Derivation and MOOSE Input Example}
\label{app:darcy}
To illustrate the kernel--PDE correspondence, consider steady Darcy flow of an incompressible fluid without gravity. The Darcy velocity $-\frac{K}{\mu}\nabla p$ is divergence-free, which gives the pressure equation
\[
-\nabla \cdot \left(\frac{K}{\mu}\nabla p\right) = 0 \quad \text{in } \Omega .
\]
Multiplying by a test function $\psi$ and integrating by parts gives the weak-form residual
\begin{equation}
\underbrace{
\int_\Omega \nabla \psi \cdot \frac{K}{\mu}\nabla p \, d\Omega
}_{\text{Kernel contribution}}
\;-\;
\underbrace{
\int_{\partial\Omega} \psi
\left(\frac{K}{\mu}\nabla p \cdot \mathbf{n}\right)\, d\Gamma
}_{\text{BC contribution}}
= 0 .
\end{equation}
The volume term is assembled by a MOOSE Kernel. The boundary term is supplied by an integrated BC object where a flux is prescribed; it vanishes on Dirichlet boundaries, where the test function is zero, and omitting it elsewhere imposes the natural condition $\frac{K}{\mu}\nabla p\cdot\mathbf{n}=0$. In the input below, the pressure is fixed on the left and right boundaries and the remaining boundaries are impermeable by this natural condition. The \texttt{DarcyFluxPressure} kernel reads $K/\mu$ from the material property \texttt{conductivity}. The relevant blocks of the input file are:

\begin{lstlisting}[language=moose, caption={Kernel, BC and material blocks of a MOOSE input for Darcy's pressure law.}, label={lst:darcy}]
[Kernels]
  [./diff]
    type = DarcyFluxPressure
    variable = pressure
  [../]
[]

[BCs]
  [./left_bc]
    type = DirichletBC
    variable = pressure
    boundary = left
    value = 1.0e6
  [../]
  [./right_bc]
    type = DirichletBC
    variable = pressure
    boundary = right
    value = 0.0
  [../]
[]

[Materials]
  [./porous]
    type = GenericConstantMaterial
    prop_names = 'conductivity'   # K/mu
    prop_values = '1e-9'          # 1e-12 m^2 / 1e-3 Pa.s
  [../]
[]
\end{lstlisting}

\section{Full System Architecture}
\label{app:architecture}

The system factors generation through an explicit physics contract and verifies the generated input by reconstructing its encoded contract. The deterministic reconstruction step is formalized in Algorithm~\ref{alg:reconstruct}.

\begin{algorithm}[h]
\caption{Deterministic PDE Reconstruction from MOOSE Input}
\label{alg:reconstruct}
\begin{algorithmic}[1]
\REQUIRE MOOSE input file AST $\mathcal{A}$, Kernel--PDE mapping $\mathcal{M}$
\ENSURE Physics contract $\mathcal{P}_{\mathrm{code}}$
\STATE $\mathcal{T} \leftarrow \emptyset$; $\mathcal{B} \leftarrow \emptyset$; $\mathcal{I} \leftarrow \emptyset$; $\mathcal{U} \leftarrow \emptyset$ \COMMENT{$\mathcal{T}$, $\mathcal{B}$, $\mathcal{I}$ are multisets}
\STATE $\mathcal{S} \leftarrow$ \textsc{Transient} or \textsc{Steady} from $\mathcal{A}$\texttt{.Executioner.type}; any other type is recorded in $\mathcal{U}$
\STATE $\mathcal{C} \leftarrow$ domain extent and dimension from a generated mesh in $\mathcal{A}$\texttt{.Mesh} (any other mesh is recorded in $\mathcal{U}$); start and end time from $\mathcal{A}$\texttt{.Executioner}
\FOR{each active kernel $\kappa$ in $\mathcal{A}$\texttt{.Kernels}, including those created by a supported Action}
    \IF{$\kappa.\texttt{type} \notin \mathcal{M}$}
        \STATE $\mathcal{U} \leftarrow \mathcal{U} \cup \{\kappa.\texttt{type}\}$; \textbf{continue} \COMMENT{unresolved}
    \ENDIF
    \STATE $(u_\kappa,o_\kappa,\alpha_\kappa,\mathbf{v}_\kappa,R_\kappa) \leftarrow \mathcal{M}[\kappa.\texttt{type}](\kappa.\texttt{params})$
    \STATE $\mathcal{T} \leftarrow \mathcal{T} \uplus \{(u_\kappa,o_\kappa,\alpha_\kappa,\mathbf{v}_\kappa,R_\kappa,\emptyset)\}$
\ENDFOR
\FOR{each material or function $\mu$}
    \STATE resolve the coefficients and constitutive facts that $\mu$ supplies to terms on its blocks; record unrecoverable values in $\mathcal{U}$ with a reason
\ENDFOR
\FOR{each BC $\beta$ / IC $\iota$}
    \STATE $\mathcal{B} \leftarrow \mathcal{B} \uplus \{(\beta.\texttt{variable}, b, \beta.\texttt{type}, \beta.\texttt{value}, \beta.\texttt{params})\}$ for each boundary $b$ in $\beta.\texttt{boundary}$; similarly for $\mathcal{I}$
\ENDFOR
\STATE record every object of a residual system that the contract does not model (DG, interface, nodal, scalar and Dirac kernels, constraints, contact) in $\mathcal{U}$
\RETURN $(\mathcal{T}, \mathcal{B}, \mathcal{I}, \mathcal{S}, \mathcal{C}, \mathcal{U})$
\end{algorithmic}
\end{algorithm}

\paragraph{Stage 1: PDE Intent Extraction.}
The extractor receives a natural-language simulation description and produces $\mathcal{P}_{\mathrm{llm}}$, a structured physics contract containing governing terms, BCs, ICs, coefficient/material facts, and temporal regime. The output is parsed as JSON before being passed downstream.

\paragraph{Stage 2: Specification-Guided Code Synthesis.}
The generator receives $\mathcal{P}_{\mathrm{llm}}$ and writes a MOOSE input file. The contract is used as a checklist for terms, BCs, ICs, materials, and temporal regime rather than as a hidden ground truth.

\paragraph{Stage 3: PDE Reconstruction.}
The verifier parses the generated input, extracts kernel, BC, IC, and material blocks, and applies the kernel--PDE mapping to reconstruct $\mathcal{P}_{\mathrm{code}}$ (\Cref{alg:reconstruct}).

\paragraph{Stage 4: Three-State Checks.}
The verifier runs the checks of \Cref{sec:method:ifs} against $\mathcal{P}_{\mathrm{llm}}$. If they pass, the input is accepted; otherwise it emits a report that lists confirmed violations and unverifiable checks separately.

\paragraph{Report-guided repair.}
The repairer receives the current code and the report and edits the code. In deployment the revised code is re-verified for up to $N_{\max}$ rounds; in the same-start comparison (\Cref{sec:exp:matched}) every arm makes exactly one repair call.

\paragraph{Prompt-level self-checking.}
Generation and refinement prompts ask the model to identify governing equations, select MOOSE objects, and check the file against the provided contract before returning the final input. This is prompt structure rather than an additional verification signal.

\paragraph{Output selection.}
The same-start comparison keeps every repair output unconditionally; no arm filters or reselects its output with the verifier, so a repair that makes the code worse is counted as such.

\section{Kernel--PDE Mapping Table}
\label{app:kernel_table}

\Cref{tab:full_kernel} provides the covered kernel--PDE weak-form mapping table used for MooseBench, and \Cref{tab:bc_mapping} the boundary conditions, including source-grounded weak-form contributions and normalized operator types.

\small
\input{full_kernel_table.tex}
\normalsize

\begin{table}[htbp]
\centering
\caption{The 16 boundary-condition classes of the mapping, the condition each imposes on a boundary $\Gamma$, and the type under which the verifier compares it. In the flux conditions, $k$ is the coefficient of the diffusion term of the equation ($k=1$ for \texttt{Diffusion}). \texttt{VacuumBC} is physically a Robin condition; the mapping files it with the flux conditions. A class outside this table is reported as unresolved.}
\label{tab:bc_mapping}
\small
\begin{tabular}{@{}lll@{}}
\toprule
\textbf{MOOSE BC} & \textbf{Mathematical Form} & \textbf{Compared as} \\
\midrule
\texttt{DirichletBC}, \texttt{ADDirichletBC} & $u = g$ & Dirichlet \\
\texttt{FunctionDirichletBC}, \texttt{ADFunctionDirichletBC} & $u = g(t,\mathbf{x})$ & Dirichlet \\
\texttt{PostprocessorDirichletBC} & $u = g_{\mathrm{pp}}$ & Dirichlet \\
\texttt{PenaltyDirichletBC} & $u = g$ (penalty) & Dirichlet \\
\texttt{ExplicitDirichletBC} & $u = g$ (explicit time integration) & Dirichlet \\
\texttt{ArrayDirichletBC} & $\mathbf{u} = \mathbf{g}$ & Dirichlet \\
\texttt{NeumannBC} & $k\,\dfrac{\partial u}{\partial n} = h$ & Neumann \\
\texttt{FunctionNeumannBC} & $k\,\dfrac{\partial u}{\partial n} = h(t,\mathbf{x})$ & Neumann \\
\texttt{VacuumBC} & $k\,\dfrac{\partial u}{\partial n} + \dfrac{\alpha}{2}\,u = 0$ & Neumann \\
\texttt{Pressure} & $\boldsymbol{\sigma}\,\mathbf{n} = -p\,\mathbf{n}$ & Neumann (traction) \\
\texttt{ConvectiveHeatFluxBC}, \texttt{ADConvectiveHeatFluxBC} & $-k\,\nabla T\cdot\mathbf{n} = h\,(T-T_\infty)$ & Robin, with $h$ and $T_\infty$ \\
\texttt{DisplacementAboutAxis} & $\mathbf{u} = \mathbf{u}_{\mathrm{rot}}(t)$ & other \\
\texttt{PorousFlowPiecewiseLinearSink} & $s=f(t,\mathbf{x})\,g(P^\beta-P_e)$ & other \\
\bottomrule
\end{tabular}
\end{table}

\section{Operator-Type Formalization}
\label{app:formalization}

This appendix formalizes the two-level representation used by reconstruction and the three-state checks of \Cref{sec:method:ifs}: normalized PDE operator types are the comparison keys, while weak-form descriptors are structural witnesses for auditing the mapping table. The propositions below hold by construction of the contracts and checks. They relate two contracts; they do not concern the numerical solution, which \Cref{sec:exp:verifier_audit,sec:exp:direct_numerics} examine empirically.

\begin{definition}[Operator sets]
\label{def:operator_sets}
\begin{align}
\mathcal{O}_{\mathrm{trial}} &= \{\texttt{identity}, \texttt{grad}, \texttt{div}, \texttt{curl}, \texttt{ddt}, \texttt{ddt2}, \texttt{stress}, \texttt{strain}\} \\
\mathcal{O}_{\mathrm{test}} &= \{\texttt{identity}, \texttt{grad}, \texttt{div}, \texttt{curl}\} \\
\mathcal{G} &= \{\texttt{scalar\_mult}, \texttt{dot}, \texttt{double\_dot}\}
\end{align}
$\mathcal{O}_{\mathrm{trial}}$ describes operators on the trial-side quantity, $\mathcal{O}_{\mathrm{test}}$ describes operators on the test function, and $\mathcal{G}$ describes the contraction mode.
\end{definition}

\begin{definition}[Normalized PDE operator type and descriptor projection]
\label{def:operator_type_projection}
Let $\mathcal{O}_{\mathrm{PDE}}$ be the finite inventory of normalized PDE operator types in the kernel--PDE mapping table. When an operator type has a weak-form witness, the table records the partial descriptor projection
\begin{equation}
d:\mathcal{O}_{\mathrm{PDE}}\rightharpoonup
\mathcal{O}_{\mathrm{trial}}\times\mathcal{O}_{\mathrm{test}}\times\mathcal{G}.
\end{equation}
Here $\rightharpoonup$ denotes a partial function: $d(o)$ is defined only when the operator type $o$ has an associated weak-form descriptor in the mapping table.
\end{definition}

\begin{remark}[Role of the descriptor]
\label{rem:descriptor_role}
The operator type $o\in\mathcal{O}_{\mathrm{PDE}}$ is the term-comparison key; the descriptor $d(o)$ is its weak-form witness. Together they make the kernel--PDE mapping table auditable: operator-type assignments are justified by the corresponding weak-form contributions recorded in the table, while the checks use normalized operator-type equality.
\end{remark}

\begin{definition}[Bound PDE term]
\label{def:bound_pde_term}
\begin{equation}
\mathrm{BoundTerm} := (u,\; o,\; \alpha,\; \mathbf{v},\; R,\; \mathbf{m}), \qquad
u\in\mathrm{Variables},\; o\in\mathcal{O}_{\mathrm{PDE}},\;
\alpha\in\mathbb{R}\cup\mathcal{E}\cup\{\bot\}.
\end{equation}
Here $u$ is the variable whose equation receives the contribution, $\mathbf{v}$ the coupled variables, $R$ the blocks on which the term acts, and $\mathbf{m}$ the constitutive facts supplied by materials. The coefficient $\alpha$ is a real number, a symbolic expression in the set $\mathcal{E}$ of expression strings, or $\bot$ when no value is recorded. A $\bot$ with a recorded reason means that a value was present but could not be recovered, and the coefficient check is then unresolved. A $\bot$ without a reason means that the contract states no coefficient, because the request leaves it open or $\mathcal{M}$ has no coefficient rule for the class; on the reference side it induces no coefficient check. $\bot$ is never read as $1$: a kernel whose residual has no coefficient, such as \texttt{Diffusion} or \texttt{TimeDerivative}, is recorded with $\alpha=1$. Descriptor metadata remains attached to $o$ through the mapping table; it is not part of the runtime term key.
\end{definition}

\begin{definition}[Kernel--PDE mapping]
\label{def:kernel_weakform_mapping}
$\mathcal{M}$ maps each covered MOOSE kernel class deterministically to a normalized PDE operator type, coefficient extraction rule, severity weight, and optional descriptor witness. Applied to a concrete kernel block $\kappa$, it yields a bound term $(u_\kappa,o_\kappa,\alpha_\kappa,\mathbf{v}_\kappa,R_\kappa,\mathbf{m}_\kappa)$, where $\alpha_\kappa$ and $\mathbf{m}_\kappa$ are resolved through the materials and functions of the input. The paper table documents 762 of 766 MooseBench ground-truth kernel-term instances (99.48\%) with source-grounded weak-form contributions; BC and IC maps contain 16 and 21 entries, respectively.
\end{definition}

\begin{definition}[Descriptor structural equivalence]
\label{def:structural_equivalence}
For available descriptors $d_a=(\tau_a,\phi_a,\gamma_a)$ and $d_b=(\tau_b,\phi_b,\gamma_b)$,
\begin{equation}
d_a \sim_s d_b \iff \tau_a = \tau_b \wedge \phi_a = \phi_b \wedge \gamma_a = \gamma_b .
\end{equation}
\end{definition}

\begin{proposition}
\label{prop:structural_equivalence_relation}
$\sim_s$ is an equivalence relation on available weak-form descriptors.
\end{proposition}

\begin{proof}
$\sim_s$ is the conjunction of equality on three descriptor components.
\end{proof}

\begin{remark}[Audit-only structural relation]
\label{rem:descriptor_audit_only}
$\sim_s$ is used only for descriptor-level auditing; the checks use operator-type equality on $\mathcal{O}_{\mathrm{PDE}}$ and do not depend on $\sim_s$.
\end{remark}

\begin{definition}[Tolerant match]
\label{def:coefficient_tolerant_match}
For a reference value $a$ and a candidate value $b$, both real,
\begin{equation}
a \approx_{\delta} b
\iff
|a-b| \leq \delta\,\max(|a|,\epsilon_0),
\end{equation}
with relative tolerance $\delta=10^{-6}$ in all reported results and $\epsilon_0=10^{-12}$, so that a zero reference admits only $|b|\leq\delta\epsilon_0$. Symbolic expressions match only when they are textually identical. Two bound terms match, $t_{\mathrm{ref}} \approx_{\delta} t_{\mathrm{cand}}$, when they agree on the aligned variable $u$, the operator type $o$, the blocks $R$ and the coupled variables $\mathbf{v}$, and $\alpha_{\mathrm{ref}} \approx_{\delta} \alpha_{\mathrm{cand}}$. The same value predicate compares boundary values, Robin parameters, initial values and constitutive scalars; the context values (domain extent, start and end time) are compared with $\delta=10^{-9}$.
\end{definition}

\begin{remark}[Reference-relative matching]
\label{rem:reference_relative_matching}
The predicate $\approx_{\delta}$ is reference-relative: for $\delta>0$ it is generally neither symmetric nor transitive. This asymmetry is intentional, because candidate values are judged against reference-induced checks rather than compared as an unordered pair. The check set $\mathcal{Q}(\mathcal{P}_{\mathrm{ref}},\mathcal{P}_{\mathrm{cand}})$ contains the checks induced by reference items plus one \emph{extra} check (severity medium) for every candidate item without a partner, so it depends on both contracts. The reference value defines the tolerance scale, consistent with $\mathcal{P}_{\mathrm{ref}}$ serving as ground truth during evaluation and as the LLM-extracted contract during deployment-time refinement. An attribute that the reference leaves open ($\bot$ without a recorded reason) induces no check.
\end{remark}

\paragraph{Severity weight table.}
Each check carries a severity weight. A term-presence or condition-type check inherits the severity that the mapping table assigns to its entry, reflecting the physical impact of the error type; the temporal-regime check uses the highest weight, coupling and constitutive-model checks the high weight, and coefficient, value, Robin-parameter and extra-item checks the medium weight (\Cref{tab:severity}):

\begin{table}[htbp]
\centering
\caption{Check severity weights (six levels, expert-defined), in decreasing order of weight. The level names are those of the mapping table: a missing source term (1.0) weighs more than a value-level change (0.7) despite its name.}
\label{tab:severity}
\small
\begin{tabular}{@{}lrl@{}}
\toprule
\textbf{Severity} & \textbf{Weight} & \textbf{Justification} \\
\midrule
Highest & 4.0 & Missing/wrong time derivative or inertial term (changes PDE class) \\
Very High & 3.0 & Wrong dominant operator (changes PDE structure) \\
High & 2.0 & Coupling term missing, BC/IC type wrong (breaks constraints) \\
Medium-Low & 1.0 & Source/forcing term missing (quantitative $>$ structural impact) \\
Medium & 0.7 & Value-level change with correct operator (quantitative error) \\
Low & 0.5 & Non-physics metadata issue \\
\bottomrule
\end{tabular}
\end{table}

\paragraph{Variable alignment procedure.}
Before comparison, candidate variable names are aligned to reference names. Variables with the same name are kept. If exactly one reference and one candidate variable remain, they are paired. Otherwise each remaining variable $u$ is described by its \emph{operator signature} $\sigma(u)$, the multiset of the operator types of the terms that act on $u$, and a candidate variable is paired with a reference variable only if their signatures are equal and unique on both sides. When several variables share a signature (\eg, two species with identical diffusion--reaction equations), no pairing is guessed: the checks of the affected variables are unresolved with origin \emph{alignment}. Offline evaluation against $\mathcal{P}_{\mathrm{gt}}$ first applies a finer signature that also contains the blocks of the terms and the boundary sets and condition types of each variable, again pairing only signatures that are unique on both sides and never consulting a score. All subsequent checks use the aligned names.

\paragraph{Steady initial guesses.}
If both contracts are steady, the initial conditions on solved variables are removed from $\mathcal{I}$ on both sides before normalization and recorded separately; they never enter $\mathcal{Q}$. Initial conditions on auxiliary fields stay in $\mathcal{I}$, and transient problems keep all of them. The definitions and propositions below refer to $\mathcal{I}$ after this step.

\begin{definition}[Term multiset of a contract]
\label{def:pde_term_set}
$\mathcal{T}$ is the finite multiset of the bound terms of a contract (\Cref{def:bound_pde_term}). Two objects that contribute identical terms are two elements of $\mathcal{T}$, so a duplicated kernel stays visible. The matching below acts on $\mathcal{T}$ after normalization (\Cref{sec:method:ifs}).
\end{definition}

\penalty0
\begin{definition}[Term matching function]
\label{def:term_matching}
Given reference $\mathcal{T}_{\mathrm{ref}}$ and candidate $\mathcal{T}_{\mathrm{cand}}$, terms are grouped by the key $(u,o,R)$. Within each group $g$, $\operatorname{match}_g$ is a one-to-one partial assignment from reference to candidate terms that minimizes the total weight of violated checks (unresolved checks carry a negligible cost so that ties prefer recovered values), computed with the Hungarian algorithm:
\begin{equation}
\operatorname{match}(t_{\mathrm{ref}}) =
\begin{cases}
t_{\mathrm{cand}} & \text{if } t_{\mathrm{ref}} \text{ is assigned to } t_{\mathrm{cand}} \text{ in its group}, \\
\bot & \text{otherwise (missing term).}
\end{cases}
\end{equation}
Extra terms are the elements of $\mathcal{T}_{\mathrm{cand}}$ outside the range of $\operatorname{match}$, counted with multiplicity. Boundary conditions are matched the same way under the key (variable, boundary) and initial conditions under the key (variable). A missing or extra item is violated unless the variable alignment is ambiguous or an unmapped object on either side could account for it, in which case it is unresolved.
\end{definition}

\begin{proposition}[Reconstruction]
\label{thm:reconstruction_soundness}
Let $c$ be a syntactically valid MOOSE input and $\mathcal{P}_{\mathrm{code}}=\operatorname{reconstruct}(c,\mathcal{M})$ (\Cref{alg:reconstruct}). Then: (a) \emph{coverage}: every active kernel object whose class is in $\mathcal{M}$, and every kernel that a supported Action adds, yields exactly one bound term; (b) \emph{fidelity}: each term carries the operator type that $\mathcal{M}$ assigns to its class and, where $\mathcal{M}$ has a coefficient rule, either the value that $c$ supplies or $\bot$ with the reason it was not recovered; (c) \emph{no spurious terms}: $|\mathcal{T}|$ equals the number of these kernels; (d) \emph{no silent omission}: $\mathcal{U}$ records every active kernel whose class is outside $\mathcal{M}$ or has no normalized operator type, every Action that is not expanded, every object of a residual system that the contract does not model (discontinuous Galerkin, interface, nodal, scalar and Dirac kernels, constraints and contact), every boundary or initial condition of an unmapped class, every function or material property that cannot be resolved, and an unsupported mesh or executioner.
\end{proposition}

\begin{proof}[Proof sketch]
By inspection of \Cref{alg:reconstruct}: it visits every active kernel block, adds exactly one bound term per covered kernel, and writes everything else that can contribute to the residual to $\mathcal{U}$. (Normalization before comparison may later sum linear terms with identical bindings; that step acts on the contract, not on reconstruction.)
\end{proof}

The proposition is relative to $\mathcal{M}$: reconstruction applies the table faithfully, and the contract states the weak form that MOOSE assembles to the extent that the table is right. The descriptor witnesses and the source-grounded entries of Appendix~\ref{app:kernel_table} make the table auditable, and the MOOSE replays of \Cref{sec:exp:verifier_audit} test it empirically.

\begin{proposition}[Characterization of $S=1$]
\label{thm:ifs_soundness}
Let $\mathcal{P}_{\mathrm{ref}}$ contain at least one term or condition, so that $S$ is defined. Then $S(\mathcal{P}_{\mathrm{ref}}, \mathcal{P}_{\mathrm{cand}}) = 1$ (\ie, $\mathrm{IFS}=1$) if and only if every check in $\mathcal{Q}$ is satisfied, and this holds if and only if, after normalization and alignment: (a) there is a bijection between $\mathcal{T}_{\mathrm{ref}}$ and $\mathcal{T}_{\mathrm{cand}}$ that preserves the variable, the operator type and the blocks; (b) paired terms have the same coupled variables, and every coefficient and constitutive fact that the reference states is met under $\approx_{\delta}$; (c) boundary and initial conditions are in bijection under their keys, with equal types and with every value and Robin parameter that the reference states met under $\approx_{\delta}$; (d) the temporal regime that the reference states is that of the candidate; (e) no reference attribute is $\bot$ with a reason, and no reference item has an unrepresented operator type or condition class.
\end{proposition}

\begin{proof}
All weights are positive, so the satisfied weight equals the total exactly when no check is violated or unresolved. An unpaired reference item induces a missing check and an unpaired candidate item an extra check, neither of which is ever satisfied; hence all checks are satisfied only if the assignment is a bijection in every group, and the remaining checks are the predicates in (b)--(e). Conversely, a pairing that meets (a)--(e) has zero cost, and every assignment of zero cost contains only satisfied checks, so the minimum-cost assignment satisfies every check.
\end{proof}

\begin{proposition}[No silent difference]
\label{thm:representational_completeness}
Under the same assumption, suppose that after normalization and alignment a reference item has no partner, a candidate item has no partner, or a pair differs in its type, its coupled variables, or a coefficient, value, Robin parameter or constitutive fact that the reference states. Then at least one check is violated or unresolved, and $S(\mathcal{P}_{\mathrm{ref}}, \mathcal{P}_{\mathrm{cand}}) < 1$. The check is violated, and not merely unresolved, whenever the difference is decidable: both values were recovered as numbers or as model names, the variable alignment is unambiguous, and neither side has an unmapped object that could supply or require the item. A difference in an attribute that the reference does not state induces no check. Such attributes, and the context $\mathcal{C}$, whose checks are reported separately, are the only ones in which two normalized, aligned contracts can differ while $S=1$.
\end{proposition}

\begin{proof}
The first claim is the contrapositive of \Cref{thm:ifs_soundness}. For the second, a check is unresolved only if a value was not recovered, a symbolic expression is compared, the alignment is ambiguous, or an unmapped object on either side could account for the item; when none of these holds, a difference fails its predicate and the check is violated. For the third, every attribute that a reference item states induces a check, and so does every unpaired candidate item.
\end{proof}

\begin{proposition}[Invariance]
\label{prop:invariance}
Let $c'$ be obtained from an input $c$ by one of the following rewrites, which leave the residual of \cref{eq:residual_sum} and the initial state unchanged up to the names of variables and properties:
(i)~reordering blocks, or the objects within a block;
(ii)~splitting a diffusion, reaction or time-derivative kernel with a constant coefficient into several kernels on the same variable, blocks and coupled variables whose coefficients have the same sum, or merging such kernels;
(iii)~replacing a boundary condition on a list of boundaries by one condition per boundary, or the converse;
(iv)~adding or removing a zero-flux Neumann condition on a variable whose terms are all diffusion, reaction, source or time-derivative terms, at least one of them diffusion, or adding or removing an explicit zero constant initial condition;
(v)~supplying a constant coefficient through a constant material property that is defined once on the blocks of the object, or the converse, renaming a material property consistently, or adding a constant property that no object reads;
(vi)~renaming a variable consistently to an unused name, provided the alignment still pairs it with the same reference variable.
Then, for every reference $\mathcal{P}_{\mathrm{ref}}$, the checks of $\mathcal{P}_{\mathrm{code}}(c')$ against $\mathcal{P}_{\mathrm{ref}}$ have the same outcomes as those of $\mathcal{P}_{\mathrm{code}}(c)$, and therefore the same $S$, $V$ and verdict.
\end{proposition}

\begin{proof}[Proof sketch]
In each case the two inputs yield contracts with the same normal form up to the compared attributes, and the comparison is a function of the normal forms. (i) A contract is a collection of multisets, and normalization, grouping and the assignment do not depend on the order of their elements. (ii) Normalization sums the constant coefficients of linear terms with identical bindings, which is sound because the residual is additive over kernels and these terms are linear in their coefficient. (iii) Reconstruction emits one condition per boundary name. (iv) Normalization removes both items: a zero-flux Neumann condition adds a zero boundary integral, and MOOSE initializes a variable to zero when no initial condition is given. (v) Reconstruction resolves a property name to its value and compares values only. (vi) Alignment maps the new name to the same reference name. In floating-point arithmetic the sum in (ii) is correctly rounded and can differ from the original coefficient in the last place, about ten orders of magnitude inside $\delta$; a sum is dropped only when it is exactly zero.
\end{proof}

The audit of Appendix~\ref{app:verifier_audit} tests rewrites (i)--(iii), (v) and (vi) on 962 rewritten inputs and \Cref{thm:representational_completeness} on 1{,}313 mutated inputs (\Cref{tab:verifier_audit_types}).

\begin{corollary}[What a verdict certifies]
\label{cor:representation_relative_guarantee}
Within the covered MOOSE fragment: (a) an input is reported \emph{satisfied} only if every check and every context check is satisfied, $\mathcal{U}$ is empty on both sides and every candidate variable is aligned, so an input with an object outside $\mathcal{M}$ is never reported satisfied; (b) every \emph{confirmed violation} names a reference item and a candidate item, or the absence of one of them, whose recovered contents differ; (c) every difference between the normalized contracts in an attribute that the reference states is reported as violated or unresolved, never as satisfied; (d) verdicts are invariant under the rewrites of \Cref{prop:invariance}.
\end{corollary}

These are statements about two contracts. They carry over to the problem that an input encodes to the extent that $\mathcal{M}$ is correct, and to the problem that was intended to the extent that $\mathcal{P}_{\mathrm{ref}}$ states it; neither is established by the propositions, and both are examined empirically in \Cref{sec:experiments}.

\begin{remark}[Severity weight design]
\label{rem:severity_weight_design}
The six-level severity assignment (highest: 4.0 down to low: 0.5) encodes a domain-expert prior over typical physical impact: errors changing PDE \emph{type} receive the largest weights; errors changing coefficient \emph{values} receive lower structural-fidelity weights. The numerical values are engineering choices rather than universal physical constants; their relative ordering reflects typical impact severity in multiphysics practice. The severity of each checkpoint is inherited from the normalized operator type and checkpoint dimension in the domain-expert-annotated mapping table (Appendix~\ref{app:kernel_table}). Verdicts are defined by the outcomes of the checks alone: the weights enter only the scores $S$ and $V$ and the choice among pairings within a group.
\end{remark}

\section{Reproducibility Manifest}
\label{app:repro}

This manifest (\Cref{tab:repro-full}) lists the parameters needed to reproduce all experiments. The supplementary archive contains the benchmark, the verifier and the run directories, each with its analysis outputs and, except for the graded-correction summary, its manifest, frozen commits and model identity.

\begin{table}[htbp]
\centering
\caption{Reproducibility parameters. Commit identifiers refer to the research repository; the verifier source is included in the archive.}
\label{tab:repro-full}
\small
\begin{tabular}{@{}lp{0.68\linewidth}@{}}
\toprule
\textbf{Item} & \textbf{Value} \\
\midrule
Benchmark & MooseBench, 220 cases (260 drafted, 40 excluded by a fixed list) \\
References & revised expert references (38 reaction terms corrected in 34 cases) \\
Verifier & current rule set (with the steady initial-guess rule), frozen at commit \texttt{e32e3f29}; relative tolerance $10^{-6}$ \\
Same-start protocol & arms Initial, Self, Runtime, Report; one repair call each; references used only in offline evaluation, except by the known-specification and graded-correction controls; frozen object-realization registry \\
API models & \texttt{gpt-5.4-2026-03-05}, \texttt{gpt-5.4-mini-2026-03-17}; undated \texttt{gpt-6-luna}, \texttt{gpt-6-sol}, \texttt{gpt-6.1-terra}; reasoning effort none, temperature 0, 8{,}192 output tokens; called 2026-11-13 and 2026-11-14 \\
Local models & \texttt{Qwen/Qwen3-14B} (revision \texttt{6b837d17}), \texttt{Qwen/Qwen3-32B} (revision \texttt{d64d397e}); non-thinking mode, BF16, vLLM 0.20.2 on A100-SXM4-80GB, temperature 0, 8{,}192 output tokens \\
Run commits & GPT-5.4 \texttt{e85529cd}; GPT-5.4 mini \texttt{56f1caee}; GPT-6 Sol \texttt{eef102e3}; GPT-6 Luna and GPT-6.1 Terra \texttt{6d1e6c86}; Qwen3 \texttt{26a6c07e} (model list frozen at \texttt{7fdf5af3}, verifier unchanged); known-specification control \texttt{f1887218}; graded correction \texttt{fb9e7eb6} \\
Direct outputs & single-call generation from the request only, temperature 0, generated 2026-05-01 to 2026-05-04; 16 of 660 outputs with a recorded syntax cleanup \\
MOOSE & \texttt{combined-opt}, application version 541a9f6f, one workstation for all reported execution checks and solves \\
Execution check & exit 0, solve start before the first error, or no error within 2\,s; unused parameters allowed \\
Numerical subset & fixed 98 cases and sample points; 90\,s limit per solve \\
Repair outcomes & 9{,}380 scored outputs: 4{,}620 in the same-start comparison (7 models $\times$ 220 cases $\times$ 3 repair arms), 1{,}760 in the known-specification control (4 $\times$ 220 $\times$ 2) and 3{,}000 new outputs of the graded correction (steps that reuse earlier outputs excluded); failed or missing outputs count as IFS 0 \\
Statistics & paired differences; percentile bootstrap over reference groups (5{,}000 draws for the same-start comparison, 10{,}000 for the numerical comparison), fixed seed \\
\bottomrule
\end{tabular}
\end{table}

\subsection{Object-Realization Registry}
\label{app:registry-note}

In the same-start comparison, every generated input passes through a deterministic layer before it is executed and verified; the layer is frozen and identical in all arms (\Cref{sec:exp:matched}). It maps the generated blocks onto MOOSE's object registry: it renames parameters to their registered names (\eg, \texttt{thermal\_conductivity} to \texttt{diffusion\_coefficient} for \texttt{HeatConduction}), switches object types to their registered variants, copies required coupled variables such as displacements from \texttt{GlobalParams}, turns numeric parameters into material properties with the same value, supplies zero temperature derivatives for constant coefficients, removes duplicate initial conditions, and adds solver tolerances. A few rewrites can change a stated condition: across the 6{,}160 outputs of the seven models and four arms, it keeps one of two conflicting initial conditions 20 times, and it rewrites a Robin condition in convective form 47 times, 22 of them with a unit transfer coefficient where it does not recognize one in the input. Its only input is the code: it reads neither the extracted contract nor the reference, and it gives no feedback to the model. Direct outputs do not pass through it.

\FloatBarrier

\section{Prompt Templates}
\label{app:prompts}

The prompts below are copied verbatim from the saved request records; placeholders are in angle brackets, and one em dash is typeset as two hyphens. Every model receives the same prompts; only the request varies per case.

\paragraph{Direct generation (\cref{sec:exp:runnable_wrong,sec:exp:direct_numerics}).} One call per case, from the request alone. System prompt:
\begin{lstlisting}[basicstyle=\ttfamily\tiny, breaklines=true, frame=single, backgroundcolor=\color{codebg}, columns=fullflexible, keepspaces=true]
You are an expert MOOSE framework developer. Generate a complete, valid MOOSE input file (.i format) for the given physics simulation request.

Requirements:
- Output ONLY the .i file content, no explanations
- Include all required blocks: [Mesh], [Variables], [Kernels], [BCs], [Materials], [Executioner], [Outputs]
- Use correct MOOSE syntax (HIT format)
- CRITICAL -- coefficients: extract every numerical parameter value (diffusivity, conductivity, viscosity, Young's modulus, Poisson's ratio, permeability, etc.) stated in the description and use it VERBATIM in [Materials]. Never substitute 1.0 or any default value when a specific value is given.
\end{lstlisting}
User prompt:
\begin{lstlisting}[basicstyle=\ttfamily\tiny, breaklines=true, frame=single, backgroundcolor=\color{codebg}, columns=fullflexible, keepspaces=true]
Generate a MOOSE input file for the following simulation:

<request>

Output the complete .i file:
\end{lstlisting}

\paragraph{Contract extraction (same-start comparison, \cref{sec:exp:matched}).} The user message is the request alone. System prompt:
\begin{lstlisting}[basicstyle=\ttfamily\tiny, breaklines=true, frame=single, backgroundcolor=\color{codebg}, columns=fullflexible, keepspaces=true]
Extract the physics explicitly requested by the user into JSON. Preserve variable names, couplings, initial conditions and numerical values. Use null for a scalar that is not specified; do not invent a value. List missing required information in unspecified_fields. time_regime is steady, transient, or null, not a numerical integrator. For 2D/3D bounds add domain.ymin/ymax and domain.zmin/zmax to context. operands lists coupled input variables; region lists explicitly specified block IDs. Coefficients follow normalized MOOSE residual conventions, not the sign of an arbitrarily rearranged strong equation: diffusion D means +(D grad(u), grad(test)); reaction r means +(r*u, test), so a consuming/decay rate is positive r; time_derivative a means +(a*d(u)/dt, test). A source f means -(f, test), and source with operands=[v] means -(f*v, test). Keep coupling signs and operands explicit. Convective (Robin) boundaries use bc_type Robin, value null and parameters {"T_infinity": number, "heat_transfer_coefficient": number} with the stated values; other conditions have empty parameters. properties is empty except on stress_divergence terms, where it holds only constitutive facts the request states, using exactly these keys and values: eigenstrain_model: thermal_expansion; elasticity_tensor: isotropic; inelastic_model: isotropic_plasticity | power_law_creep; strain_model: finite | small; stress_model: elastic | elastic_pf_fracture | inelastic; numeric keys: bulk_modulus, hardening_constant, lambda, poissons_ratio, shear_modulus, stress_free_temperature, thermal_expansion_coeff, yield_stress, youngs_modulus. Omit any fact the request does not state; use no other keys or values. Normalized operators: advection, allen_cahn, cahn_hilliard, coupled_force, curl_curl, diffusion, inertia, ns_continuity, ns_pressure, ns_viscous, other, pf_darcy_flux, pf_effective_stress, reaction, source, stress_divergence, time_derivative. Use this schema and return JSON only:
{
 "variables": ["T"],
 "terms": [{"variable":"T", "operator":"diffusion", "coefficient":null,
            "operands":[], "region":[], "properties":{}}],
 "boundary_conditions": [{"variable":"T", "boundary":"left",
                          "bc_type":"Dirichlet", "value":300.0, "parameters":{}}],
 "initial_conditions": [{"variable":"T", "ic_type":"constant", "value":300.0}],
 "time_regime": "steady",
 "context": {"dimensions":1, "domain.xmin":0, "domain.xmax":1,
             "time.start":null, "time.end":null},
 "unspecified_fields": []
}
\end{lstlisting}

\paragraph{Initial generation.} System prompt, also used for every repair call:
\begin{lstlisting}[basicstyle=\ttfamily\tiny, breaklines=true, frame=single, backgroundcolor=\color{codebg}, columns=fullflexible, keepspaces=true]
You are an expert MOOSE developer. Generate a complete MOOSE input
file for the original request. Preserve all explicitly specified governing
physics, variables, boundary and initial conditions, material values, dimensions,
domain bounds and simulation time interval. Do not simplify the requested
physics to make the input easier to run. Return only the complete .i input,
optionally inside a single code fence.
\end{lstlisting}
User prompt:
\begin{lstlisting}[basicstyle=\ttfamily\tiny, breaklines=true, frame=single, backgroundcolor=\color{codebg}, columns=fullflexible, keepspaces=true]
Original request:
<request>

Extracted contract:
<contract JSON>
\end{lstlisting}

\paragraph{Repair (Self, Runtime, Report).} One call per arm, with the system prompt above. User prompt:
\begin{lstlisting}[basicstyle=\ttfamily\tiny, breaklines=true, frame=single, backgroundcolor=\color{codebg}, columns=fullflexible, keepspaces=true]
Original request:
<request>

Extracted contract:
<contract JSON>

Current input:
<initial input>

Feedback:
<feedback>
Return a complete corrected input. Preserve the requested domain, time interval, variables and physics. If no correction is needed, return the input unchanged.
\end{lstlisting}
The feedback is ``Review the input against the original request and extracted contract.'' for Self, ``MOOSE execution status and log:'' followed by the execution record as JSON for Runtime, and the verifier's report for Report. A report from the GPT-5.4 run (\texttt{coupled\_005}):
\begin{lstlisting}[basicstyle=\ttfamily\tiny, breaklines=true, frame=single, backgroundcolor=\color{codebg}, columns=fullflexible, keepspaces=true]
Deterministic comparison of the current input with the extracted contract. Confirmed violations are definite mismatches. Items the checker could not verify may already be correct; change them only if the request requires it.
Confirmed violations:
- [coefficient] diffusion for T on all blocks: expected 30.0, found 1.0
- [coupling] source for c on all blocks: expected ('T',), found ()
Could not verify:
- [coefficient] source for c on all blocks: expected 0.001; function 'c_source_fn' depends on named symbols
- [bc] T on right: expected 'Robin'; boundary class ConvectiveFluxBC is not mapped
- [bc] T on right T_infinity: expected 300.0; the value was not recovered from the code
- [bc] T on right heat_transfer_coefficient: expected 50.0; the value was not recovered from the code
\end{lstlisting}
In the known-specification control, the contract JSON is the projected expert contract and the report compares the input with the full expert contract; all other text, including the report's first sentence, is unchanged. In the graded correction, the contract JSON and the report use the partly corrected contract in the same way.

\section{Verifier Audit and Direct Re-audit Details}
\label{app:verifier_audit}

\paragraph{Transformations.}
Each of the 23 transformation types is applied to every one of the 220 expert references where it applies; applicability depends on the input and the object registry, not on the resulting verdict. A rewrite counts as expected when the score, status and context checks are unchanged; a mutation counts as expected when it is flagged as a violation in the changed category and variable. \emph{Unresolved} means only unresolved checks remain at the changed item; \emph{no recon.} means the mutated input no longer yields a reconstructable equation. The numerical replay draws, by a fixed hash within each family and transformation, from the 98-case numerical subset fixed before any generation; every drawn input stays in the denominator. Fields count as unchanged within the saved numerical tolerances. \Cref{tab:verifier_audit_types} lists all types.

\paragraph{Object comparison.}
As a baseline, we compare each rewritten input with its reference object by object: the type and parameter values of every object in the \texttt{Kernels}, \texttt{BCs}, \texttt{ICs} and \texttt{Materials} blocks, and the executioner type and time interval. Object names and declaration order are ignored, and numerical values match to a relative difference of $10^{-6}$. This comparison keeps the verdict for all 220 block reorderings and flags every other rewrite: 48 boundary-list merges, 17 splits, 39 constants moved into materials, 55 linear splits, 106 material renames, 65 materials replaced by constants, 220 added unused materials and 192 variable renames, 742 of the 962 rewrites in all.

\begin{table}[htbp]
\centering
\caption{Verifier audit by transformation type; the counts are identical with and without the steady initial-guess rule. ``Replay'' gives MOOSE re-runs as changed/unchanged/not comparable. Denominators are transformation instances, not independent problems.}
\label{tab:verifier_audit_types}
\small
\begin{tabular}{@{}llrrrrl@{}}
\toprule
\textbf{Transformation} & \textbf{Kind} & \textbf{Applic.} & \textbf{As expected} & \textbf{Unres.} & \textbf{No recon.} & \textbf{Replay $n$: ch./unch./n.c.} \\
\midrule
Block reorder & rewrite & 220 & 220 & 0 & 0 & 11: 0/11/0 \\
Boundary-list merge & rewrite & 48 & 48 & 0 & 0 & 8: 0/8/0 \\
Boundary-list split & rewrite & 17 & 17 & 0 & 0 & 7: 0/7/0 \\
Constant $\to$ material & rewrite & 39 & 39 & 0 & 0 & 6: 0/6/0 \\
Linear merge & rewrite & 0 & -- & -- & -- & -- \\
Linear split & rewrite & 55 & 55 & 0 & 0 & 7: 0/7/0 \\
Material rename & rewrite & 106 & 106 & 0 & 0 & 9: 0/9/0 \\
Material $\to$ constant & rewrite & 65 & 65 & 0 & 0 & 7: 0/7/0 \\
Unused material & rewrite & 220 & 220 & 0 & 0 & 11: 0/11/0 \\
Variable rename & rewrite & 192 & 192 & 0 & 0 & 11: 0/11/0 \\
\midrule
Boundary type & mutation & 216 & 216 & 0 & 0 & 11: 10/0/1 \\
Boundary value & mutation & 216 & 216 & 0 & 0 & 11: 10/0/1 \\
Direct coefficient $\times2$ & mutation & 55 & 55 & 0 & 0 & 7: 6/1/0 \\
Material coefficient $\times2$ & mutation & 65 & 56 & 9 & 0 & 7: 7/0/0 \\
Coupling swap & mutation & 4 & 4 & 0 & 0 & 2: 2/0/0 \\
Delete term & mutation & 150 & 132 & 0 & 18 & 11: 9/1/1 \\
Duplicate term & mutation & 150 & 150 & 0 & 0 & 11: 10/1/0 \\
Initial value & mutation & 145 & 145 & 0 & 0 & 10: 8/1/1 \\
Material on wrong block & mutation & 0 & -- & -- & -- & -- \\
Reaction sign & mutation & 46 & 46 & 0 & 0 & 6: 5/0/1 \\
Robin coefficient $\times2$ & mutation & 39 & 39 & 0 & 0 & 6: 5/1/0 \\
Swap equation coefficients & mutation & 7 & 7 & 0 & 0 & 3: 3/0/0 \\
Temporal regime & mutation & 220 & 220 & 0 & 0 & 11: 0/0/11 \\
\midrule
\textbf{Rewrites} & & \textbf{962} & \textbf{962} & 0 & 0 & 77: 0/77/0 \\
\textbf{Mutations} & & \textbf{1{,}313} & \textbf{1{,}286} & 9 & 18 & 96: 75/5/16 \\
\bottomrule
\end{tabular}
\end{table}

The 16 non-comparable mutation replays are 10 temporal-regime changes that alter the time context (a steady and a transient solution are not compared pointwise), 3 unavailable references, 1 execution failure and 2 timeouts. The 5 mutations with no resolvable field change are reported as such rather than as detections or misses.

\paragraph{Direct re-audit.}
The saved Direct outputs of the three main models (\Cref{tab:runnable_wrong}) are scored with the current verifier and re-executed unchanged with the execution check of \cref{sec:exp:setup}, with unused parameters allowed; the check is an initialization signal, not convergence (normal exits: 123, 76 and 73 for Sonnet 4.6, GPT-5.4 and DeepSeek V4 Flash). The three generation scripts pass only the natural-language request to the model; 16 outputs (3, 1 and 12) received a deterministic syntax cleanup that is recorded with the output. Verdicts are computed against the revised references (\cref{sec:benchmark}); all 220 reference files used by the re-audit match the revised set by content hash. The resulting counts are the Direct rows of \Cref{tab:runnable_wrong}.

\paragraph{What the violations are.} \Cref{tab:violation_categories} counts, among the runnable outputs with a confirmed violation, those with at least one violated check in each category. An output can fall in several categories. Physics is the union of all categories except the context (domain extent and time horizon); 173 of the 180 violating outputs have a physics violation, and only 7 violate the context alone. Coefficients and boundary conditions dominate; term and initial-condition errors follow.

\begin{table}[htbp]
\centering
\caption{Violation categories among runnable Direct outputs with a confirmed violation. Each entry counts outputs with at least one violated check of that category; categories overlap. Const.\ denotes constitutive facts.}
\label{tab:violation_categories}
\small\setlength{\tabcolsep}{4pt}
\begin{tabular}{@{}lccccccccc@{}}
\toprule
\textbf{LLM} & \textbf{Viol.} & \textbf{Physics} & \textbf{Coef.} & \textbf{BC} & \textbf{Term} & \textbf{IC} & \textbf{Const.} & \textbf{Time/coupl.} & \textbf{Context only} \\
\midrule
Claude Sonnet 4.6 & 59 & 56 & 18 & 33 & 16 & 10 & 12 & 1 & 3 \\
GPT-5.4 & 66 & 64 & 46 & 28 & 10 & 7 & 1 & 1 & 2 \\
DeepSeek V4 Flash & 55 & 53 & 40 & 22 & 4 & 7 & 1 & 0 & 2 \\
\midrule
Total & 180 & 173 & 104 & 83 & 30 & 24 & 14 & 2 & 7 \\
\bottomrule
\end{tabular}
\end{table}

\paragraph{Effect of the steady initial-guess rule.}
The verifier records, but does not check, initial values on solved variables of steady problems (\cref{sec:method:ifs}). Checking them as well would change 38 of the 660 Direct scores and 21 verdicts (16 satisfied and 5 unresolved-only outputs would become violations), and the check lists of 1 of the 220 reference self-comparisons and 6 of the 2,275 transformed inputs without changing any verdict; nothing else changes. \Cref{tab:direct_v21_v22} gives \Cref{tab:runnable_wrong} both ways.

\begin{table}[htbp]
\centering
\caption{Runnable Direct outputs without and with the steady initial-guess rule. Runnable counts do not change.}
\label{tab:direct_v21_v22}
\small
\begin{tabular}{@{}lccccccc@{}}
\toprule
& & \multicolumn{2}{c}{\textbf{Violation}} & \multicolumn{2}{c}{\textbf{Unres.\ only}} & \multicolumn{2}{c}{\textbf{Satisfied}} \\
\cmidrule(lr){3-4}\cmidrule(lr){5-6}\cmidrule(lr){7-8}
\textbf{LLM} & \textbf{Runnable} & Without & With & Without & With & Without & With \\
\midrule
Claude Sonnet 4.6 & 145 & 70 & 59 & 24 & 27 & 51 & 59 \\
GPT-5.4 & 112 & 69 & 66 & 12 & 13 & 31 & 33 \\
DeepSeek V4 Flash & 91 & 60 & 55 & 6 & 6 & 25 & 30 \\
\bottomrule
\end{tabular}
\end{table}

\section{Direct Numerical Comparison Details}
\label{app:direct_numerics}

The 98-case subset was drawn by physics family and complexity tier in a fixed hash order; the subset and its sample points were fixed before any generation. Each of the 294 Direct outputs (98 per model) and each reference is solved once with the same MOOSE build (commit 541a9f6f) with a 90-second limit. The per-case error is the largest relative error over field variables at the frozen sample points; no case needed an absolute error. The statistics were fixed before the solves: per model, a one-sided Mann--Whitney test of violation against satisfied; pooled, the median difference with a 95\% percentile bootstrap CI (10,000 draws, fixed seed) that resamples cases and keeps all outputs of a case together. Thresholds $10^{-6}$ and $10^{-2}$ are descriptive. Of the 294 outputs, 120 are completed and comparable (63 with a confirmed violation, 48 satisfied and 9 unresolved-only: 5, 3 and 1 for Sonnet~4.6, GPT-5.4 and DeepSeek V4 Flash), 92 fail, 29 use a different domain or time horizon, 44 cannot be parsed or mapped to the reference fields, and 9 have no reference solution. The steady initial-guess rule (\cref{app:verifier_audit}) was added after inspecting the low-error violations of the first numerical comparison, so \Cref{tab:direct_numerics_detail} reports both sets of results; without the rule, every per-model $p$ is at most $1.4{\times}10^{-4}$ and the pooled median difference is 1.34 (CI $[1.00, 6.11]$).

\paragraph{Violating outputs with small errors.} Two of the 63 comparable violating outputs agree with the reference to $10^{-6}$. One rescales a one-dimensional elastic bar into an equivalent diffusion equation that the verifier does not map; this is a conservative false alarm. The other sets a diffusivity through a parameter that MOOSE ignores; the error is real but has no effect because that solution has no gradient. The pooled median difference between violating and satisfied outputs is 1.93 (95\% CI $[1.00, 6.25]$); per-model statistics are in \Cref{tab:direct_numerics_detail}.

\paragraph{Satisfied outputs with large errors.} Six of the 48 comparable satisfied outputs differ from the reference by more than 1\%: four by 1--3\% and two by about 100\%. The two encode the same equations, coefficients and boundary conditions as the reference on a finer mesh (2 \vs 5\,mm and 1 \vs 2.5\,mm), but both pairs of meshes are 3 to 16 times coarser than the reaction--diffusion layer of width $\sqrt{D/k}\approx0.3$\,mm, so neither discrete solution resolves it and the two disagree at the sample points. These errors reflect resolution, not a missed violation.

\begin{table}[htbp]
\centering
\caption{Direct numerical comparison without and with the steady initial-guess rule. Median error with interquartile range; $\le$/$>$ give the counts with error at most $10^{-6}$ and above $10^{-2}$; $p$ is the one-sided Mann--Whitney test, and the pooled rows give the median difference with its clustered bootstrap 95\% CI instead.}
\label{tab:direct_numerics_detail}
\small\setlength{\tabcolsep}{3pt}
\begin{tabular}{@{}llccccccl@{}}
\toprule
& & \multicolumn{3}{c}{\textbf{Violation}} & \multicolumn{3}{c}{\textbf{Satisfied}} & \\
\cmidrule(lr){3-5}\cmidrule(lr){6-8}
\textbf{Rule} & \textbf{LLM} & $n$ & Median [IQR] & $\le$/$>$ & $n$ & Median [IQR] & $\le$/$>$ & $p$ or $\Delta$ [CI] \\
\midrule
Without & Sonnet 4.6 & 24 & 1.00 [0.10, 23.6] & 2/19 & 24 & $2.4{\times}10^{-4}$ [$5.9{\times}10^{-10}$, $4.4{\times}10^{-3}$] & 10/6 & $3.9{\times}10^{-5}$ \\
 & GPT-5.4 & 20 & 2.00 [1.00, 287] & 0/19 & 13 & $1.7{\times}10^{-4}$ [$3.4{\times}10^{-9}$, $1.2{\times}10^{-3}$] & 4/0 & $1.1{\times}10^{-6}$ \\
 & DS V4 Flash & 22 & 1.97 [0.54, 6.22] & 3/18 & 8 & $2.4{\times}10^{-9}$ [$3.1{\times}10^{-12}$, $3.0{\times}10^{-5}$] & 5/0 & $1.4{\times}10^{-4}$ \\
 & Pooled & 66 & 1.34 [0.47, 17.1] & 5/56 & 45 & $5.5{\times}10^{-5}$ [$1.4{\times}10^{-9}$, $1.1{\times}10^{-3}$] & 19/6 & 1.34 [1.00, 6.11] \\
\midrule
With & Sonnet 4.6 & 23 & 1.00 [0.44, 28.5] & 1/19 & 25 & $5.5{\times}10^{-5}$ [$5.1{\times}10^{-10}$, $2.4{\times}10^{-3}$] & 11/6 & $7.3{\times}10^{-6}$ \\
 & GPT-5.4 & 20 & 2.00 [1.00, 287] & 0/19 & 13 & $1.7{\times}10^{-4}$ [$3.4{\times}10^{-9}$, $1.2{\times}10^{-3}$] & 4/0 & $1.1{\times}10^{-6}$ \\
 & DS V4 Flash & 20 & 2.07 [1.00, 7.73] & 1/18 & 10 & $7.4{\times}10^{-10}$ [$1.6{\times}10^{-12}$, $1.6{\times}10^{-5}$] & 7/0 & $1.1{\times}10^{-5}$ \\
 & Pooled & 63 & 1.93 [0.92, 18.7] & 2/56 & 48 & $3.7{\times}10^{-5}$ [$4.0{\times}10^{-10}$, $1.1{\times}10^{-3}$] & 22/6 & 1.93 [1.00, 6.25] \\
\bottomrule
\end{tabular}
\end{table}

\FloatBarrier

\section{Same-Start Comparison Details}
\label{app:matched}

\paragraph{Protocol.} Each model makes one extraction call, one initial generation with the extracted contract, and one repair call per repair arm. The prompts, the report renderer, the checker and the statistics are fixed before the first request. Unreconstructable outputs and extraction failures stay in the 220-case denominator with IFS 0. Paired differences use a bootstrap over reference groups (5{,}000 draws, fixed seed). GPT-6 Luna, GPT-6 Sol and GPT-6.1 Terra are called by undated names; every response records the model name returned by the API.

\paragraph{Extraction accuracy.} \Cref{tab:extraction_quality} compares each model's extracted contract with the reference contract using the same checker, and places it next to the effect of the report. The ordering does not match: GPT-5.4 extracts the most accurate contracts and is among the models the report hurts most, while the Qwen3 models extract the least accurate contracts and are not hurt. Average extraction accuracy therefore does not decide whether the report helps; the regressions in \cref{sec:exp:matched} come from specific contract errors that the report turns into edits. With seven models this comparison is descriptive.

\begin{table}[htbp]
\centering
\caption{Accuracy of the extracted contract against the reference (IFS and violated share of the contract itself, 220 cases; failed extractions count as IFS 0) and the Report$-$Self difference of \Cref{tab:matched}.}
\label{tab:extraction_quality}
\small
\begin{tabular}{@{}lccc@{}}
\toprule
\textbf{LLM} & \textbf{Extraction IFS} & \textbf{Extraction violated} & \textbf{Report$-$Self} \\
\midrule
GPT-5.4 & 0.825 & 0.113 & $-$0.037 \\
GPT-6 Sol & 0.804 & 0.139 & $+$0.024 \\
GPT-6.1 Terra & 0.768 & 0.142 & $-$0.017 \\
GPT-5.4 mini & 0.694 & 0.147 & $-$0.018 \\
GPT-6 Luna & 0.661 & 0.165 & $-$0.048 \\
Qwen3-14B & 0.558 & 0.212 & $+$0.011 \\
Qwen3-32B & 0.530 & 0.158 & $+$0.024 \\
\bottomrule
\end{tabular}
\end{table}

\paragraph{Known-specification control.} The reference contract is projected into the extraction schema; fields that the schema cannot express are listed as unspecified rather than guessed. Both arms show this projection to the model; the report checks the input against the full reference contract, so it can also name facts the schema cannot express. Prompts and the report renderer are otherwise identical to the extracted-contract arms, and both arms start from the saved initial input of the original run, byte for byte. The report differs from the extracted-contract report in 136 of 218 cases for GPT-5.4. Cases without an initial input (2, 5, 9 and 2 for GPT-5.4, GPT-5.4 mini, GPT-6 Luna and GPT-6 Sol) stay in the denominator as failures and are not regenerated. Violated and unresolved shares are compared on the cases where both arms can be scored; \Cref{tab:oracle_exec} lists these denominators. \Cref{tab:oracle_spec} gives the arm-level IFS. The graded correction below makes the same comparison with one contract representation at every step.

\begin{table}[htbp]
\centering
\caption{Known-specification control: IFS of the Self and Report arms when the contract shown to the model, and checked by the report, is the model's own extraction or the expert reference (same initial inputs, 220 cases), and Report$-$Self with the reference contract (95\% bootstrap CI).}
\label{tab:oracle_spec}
\small
\begin{tabular}{@{}lccccc@{}}
\toprule
& \multicolumn{2}{c}{\textbf{Extracted}} & \multicolumn{2}{c}{\textbf{Reference}} & \\
\cmidrule(lr){2-3}\cmidrule(lr){4-5}
\textbf{LLM} & \textbf{Self} & \textbf{Report} & \textbf{Self} & \textbf{Report} & \textbf{Report$-$Self (reference)} \\
\midrule
GPT-5.4 & 0.815 & 0.778 & 0.847 & 0.884 & $+$0.038 [$+$0.021, $+$0.056] \\
GPT-5.4 mini & 0.625 & 0.607 & 0.754 & 0.789 & $+$0.035 [$+$0.020, $+$0.052] \\
GPT-6 Luna & 0.763 & 0.715 & 0.823 & 0.841 & $+$0.018 [$+$0.004, $+$0.030] \\
GPT-6 Sol & 0.816 & 0.840 & 0.886 & 0.911 & $+$0.025 [$+$0.016, $+$0.034] \\
\bottomrule
\end{tabular}
\end{table}

\begin{table}[htbp]
\centering
\caption{Known-specification control: execution check of all repaired outputs (inputs that pass, out of 220) and the paired differences Report$-$Self with the reference contract in the violated ($\Delta V$) and unresolved ($\Delta U$) shares, with 95\% bootstrap CIs over reference groups and the number of cases scorable in both arms.}
\label{tab:oracle_exec}
\small\setlength{\tabcolsep}{3.5pt}
\begin{tabular}{@{}lccccccll@{}}
\toprule
& \multicolumn{5}{c}{\textbf{Runnable (of 220)}} & & & \\
\cmidrule(lr){2-6}
& & \multicolumn{2}{c}{\textbf{Extracted}} & \multicolumn{2}{c}{\textbf{Reference}} & & & \\
\cmidrule(lr){3-4}\cmidrule(lr){5-6}
\textbf{LLM} & \textbf{Initial} & \textbf{Self} & \textbf{Report} & \textbf{Self} & \textbf{Report} & $\boldsymbol{n}$ & $\boldsymbol{\Delta V}$ & $\boldsymbol{\Delta U}$ \\
\midrule
GPT-5.4 & 132 & 138 & 143 & 140 & 152 & 215 & $-$0.039 [$-$0.054, $-$0.025] & $+$0.002 [$-$0.009, $+$0.016] \\
GPT-5.4 mini & 78 & 80 & 85 & 75 & 76 & 192 & $-$0.028 [$-$0.046, $-$0.012] & $-$0.007 [$-$0.021, $+$0.005] \\
GPT-6 Luna & 110 & 111 & 115 & 114 & 117 & 207 & $-$0.021 [$-$0.031, $-$0.013] & $-$0.002 [$-$0.008, $+$0.004] \\
GPT-6 Sol & 143 & 146 & 154 & 148 & 153 & 216 & $-$0.018 [$-$0.027, $-$0.009] & $-$0.007 [$-$0.016, $+$0.000] \\
\bottomrule
\end{tabular}
\end{table}

\paragraph{Check weights.} IFS weights checks by the severity table of \cref{app:formalization}. To test whether the conclusions depend on it, we recompute every same-start difference with all weights equal, from the saved check outcomes and with the same bootstrap. \Cref{tab:weight_sensitivity} lists the Report$-$Self differences. No difference changes sign. With the extracted contract, the differences for GPT-5.4 mini and GPT-6 Sol lose significance and the one for GPT-6.1 Terra gains it; GPT-5.4 and GPT-6 Luna remain significantly negative and Qwen3-32B significantly positive. With the reference contract, all four gains remain significant and grow. The released tables give all 165 differences (five comparisons per model and contract, each in $S$, $V$ and $U$).

\begin{table}[htbp]
\centering
\caption{Report$-$Self in IFS under the severity weights used in the paper and with all check weights equal (95\% bootstrap CIs over reference groups, 220 cases).}
\label{tab:weight_sensitivity}
\small
\begin{tabular}{@{}llcc@{}}
\toprule
\textbf{Contract} & \textbf{LLM} & \textbf{Severity weights} & \textbf{Equal weights} \\
\midrule
Extracted & GPT-5.4 & $-$0.037 [$-$0.058, $-$0.017] & $-$0.030 [$-$0.055, $-$0.008] \\
 & GPT-5.4 mini & $-$0.018 [$-$0.031, $-$0.003] & $-$0.011 [$-$0.024, $+$0.004] \\
 & GPT-6 Luna & $-$0.048 [$-$0.070, $-$0.027] & $-$0.038 [$-$0.059, $-$0.017] \\
 & GPT-6 Sol & $+$0.024 [$+$0.003, $+$0.045] & $+$0.023 [$-$0.002, $+$0.047] \\
 & GPT-6.1 Terra & $-$0.017 [$-$0.037, $+$0.003] & $-$0.023 [$-$0.046, $-$0.001] \\
 & Qwen3-14B & $+$0.011 [$-$0.009, $+$0.030] & $+$0.003 [$-$0.014, $+$0.019] \\
 & Qwen3-32B & $+$0.024 [$+$0.009, $+$0.044] & $+$0.023 [$+$0.009, $+$0.042] \\
\midrule
Reference & GPT-5.4 & $+$0.038 [$+$0.021, $+$0.056] & $+$0.056 [$+$0.037, $+$0.076] \\
 & GPT-5.4 mini & $+$0.035 [$+$0.020, $+$0.052] & $+$0.040 [$+$0.024, $+$0.058] \\
 & GPT-6 Luna & $+$0.018 [$+$0.004, $+$0.030] & $+$0.028 [$+$0.015, $+$0.042] \\
 & GPT-6 Sol & $+$0.025 [$+$0.016, $+$0.034] & $+$0.036 [$+$0.026, $+$0.048] \\
\bottomrule
\end{tabular}
\end{table}

\paragraph{Graded correction.} For each model and case, the checker's alignment lists the entries in which the extracted contract disagrees with the reference contract: a term, boundary or initial condition, coefficient, temporal regime or dimension that is missing, extra or different. Each disagreement is one correction; a case has $m$ of them (median 2; 58, 47 and 55 cases have none for GPT-5.4, GPT-5.4 mini and GPT-6 Luna). Entries that the alignment cannot match uniquely, such as the components of a vector unknown, stay as extracted (17, 43 and 44 cases). The corrections of a case are put in a fixed random order, and the three steps apply the first $\lceil m/3\rceil$, the first $\lceil 2m/3\rceil$ and all $m$, so each step contains the previous one. In one case the first correction would remove the only term and leave an invalid contract; there the next seeded order that keeps every step valid is used. Cases with $m=0$ reuse the extracted-contract outputs at every step. All steps use the same projection and report renderer, unlike the reference arms of \Cref{tab:oracle_spec}, whose report checks against the full reference contract. Each step reruns the Self and Report arms from the saved initial input, 2{,}904 new repair calls in all (plus 96 outputs for cases without an initial input, scored as IFS~0, for the 3{,}000 new scored outputs of \Cref{tab:repro-full}). The statistic is the per-case least-squares slope of Report$-$Self over the corrected share (0, $\nicefrac13$, $\nicefrac23$, 1). The pooled slope is $+$0.066 [$+$0.052, $+$0.082] (\Cref{tab:spec_dose}). Runnable repaired inputs (of 220, Self/Report) are 138/143 with the extracted contract and 141/153 with every correctable entry fixed for GPT-5.4, 80/85 and 80/80 for GPT-5.4 mini, and 111/115 and 112/116 for GPT-6 Luna.

\begin{table}[htbp]
\centering
\caption{Graded correction: IFS of the Self and Report arms as a growing share of the correctable errors of the extracted contract is fixed (220 cases per model), and the slope of Report$-$Self over the corrected share (95\% bootstrap CI over reference groups).}
\label{tab:spec_dose}
\small
\begin{tabular}{@{}llccccl@{}}
\toprule
& & \multicolumn{4}{c}{\textbf{Share corrected}} & \\
\cmidrule(lr){3-6}
\textbf{LLM} & \textbf{Arm} & \textbf{0} & \textbf{1/3} & \textbf{2/3} & \textbf{1} & \textbf{Slope} \\
\midrule
GPT-5.4 & Self & 0.815 & 0.818 & 0.817 & 0.826 & \\
 & Report & 0.778 & 0.818 & 0.843 & 0.864 & $+$0.076 [$+$0.053, $+$0.101] \\
GPT-5.4 mini & Self & 0.625 & 0.641 & 0.658 & 0.664 & \\
 & Report & 0.607 & 0.666 & 0.699 & 0.711 & $+$0.063 [$+$0.043, $+$0.084] \\
GPT-6 Luna & Self & 0.763 & 0.772 & 0.777 & 0.780 & \\
 & Report & 0.715 & 0.751 & 0.776 & 0.792 & $+$0.060 [$+$0.043, $+$0.078] \\
\bottomrule
\end{tabular}
\end{table}

\paragraph{Execution environment.} All execution checks reported in the paper run on one workstation with the same MOOSE build (2-second window, unused parameters allowed). In the final audit of the 4{,}400 outputs behind \Cref{tab:oracle_exec}, every input that passes by reaching the time limit has started and logged its solve; the same holds for the graded-correction runs. For the locally served Qwen3 models, part of the runtime logs given to the Runtime arm were empty, so that arm received less information for these two models.

\section{MooseBench Benchmark Details}
\label{app:moosebench}

MooseBench contains 220 curated prompts (260 total, 40 excluded) spanning 7 expert-defined physics families and 3 complexity tiers:
\begin{itemize}
    \item \textbf{Simple} ($\leq 4$ checkpoints): Single-physics, minimal configuration (e.g., steady-state diffusion with two BCs).
    \item \textbf{Medium} (5--8 checkpoints): Single-physics with richer configuration or simple coupling.
    \item \textbf{Complex} ($\geq 9$ checkpoints): Coupled multiphysics, multiple variables, or sophisticated boundary conditions.
\end{itemize}

\paragraph{Family balance.}
The heat-transfer families are larger because they supply many stable single- and multi-physics MOOSE examples with clean PDE-level ground truth and executable source files.

\begin{table}[htbp]
\centering
\caption{Canonical 220-case MooseBench family distribution.}
\label{tab:stratified}
\small
\begin{tabular}{@{}lrc@{}}
\toprule
\textbf{Physics Family} & \textbf{Count} & \textbf{\%} \\
\midrule
Transient Heat Conduction       & 66 & 30.0 \\
Steady-State Heat Conduction    & 36 & 16.4 \\
TH-Coupled Porous Media Flow    & 30 & 13.6 \\
Linear Elasticity               & 28 & 12.7 \\
THM-Coupled Porous Media Flow   & 20 &  9.1 \\
Plasticity                      & 20 &  9.1 \\
Phase-Field                     & 20 &  9.1 \\
\midrule
\textbf{Total}                  & \textbf{220} & \textbf{100} \\
\bottomrule
\end{tabular}
\end{table}

\paragraph{Prompt design.}
All prompts describe physics in domain-expert language without referencing MOOSE-specific constructs (kernel names, BC type names, variable naming conventions).

Complex-tier prompts combine multi-variable coupling (up to 3 fields), specialized MOOSE kernels (\texttt{HeatConduction}, \texttt{MatDiffusion}, \texttt{AnisotropicDiffusion}, \texttt{CoupledForce}, \texttt{CoefReaction}), mixed boundary condition types (Dirichlet, Neumann and convective) and coupled engineering problems (nuclear fuel pellet, cascade reaction A$\to$B$\to$C, thermal shock with hydrogen permeation). The largest case (\texttt{coupled\_007}) involves 10 kernels, 3 variables and 7 numerical coefficients.

\paragraph{Ground-truth construction.}
For each prompt, a reference MOOSE input file implements the intended physics. The expert contract $\mathcal{P}_{\mathrm{gt}}$ is obtained from it by the reconstruction of \Cref{alg:reconstruct}, so it uses the same mapping $\mathcal{M}$ as the verifier. Kernel classes that the mapping normalizes to the same PDE operator type (e.g., \texttt{Diffusion}, \texttt{ADDiffusion} and \texttt{HeatConduction} for the diffusion operator) are therefore interchangeable.

All 220 cases pass self-validation: $\mathrm{IFS}(\mathcal{P}_{\mathrm{gt}}, \mathcal{P}_{\mathrm{gt}}) = 1.0$ with zero MOOSE-term leakage in the natural-language descriptions. The prompts, reference inputs and mapping annotations were written and finalized by the authors, with some reference inputs adapted from public MOOSE examples; no human subjects were involved.

\paragraph{Revision.}
One revision changed the references: 38 reaction terms in 34 references had the wrong sign under MOOSE's residual convention, in which a positive multiplier is consumption. Each was checked against the decay or production its prompt requests and corrected; all verdicts in the paper use the revised set.

\paragraph{Example prompts (one per tier).}

\emph{Simple (thermal\_001):} ``A thin membrane separates two regions with different concentrations of a dissolved substance. The bottom side is maintained at concentration 1, the top side at concentration 0. Assuming uniform diffusivity, find the steady-state concentration distribution through the membrane.'' \textbf{GT:} 1 term (diffusion), 2 Dirichlet BCs, steady.

\emph{Medium (diffusion\_004):} ``A substance with diffusivity 0.1 m$^2$/s spreads through a 2D square region (1 m $\times$ 1 m). The left boundary is held at zero concentration. A constant mass flux of 1 mol/(m$^2$s) enters from the right. The top and bottom are insulated. Starting from zero concentration everywhere, simulate the transient evolution for 2 seconds using dt $=$ 0.1 s.'' \textbf{GT:} 2 terms (diffusion + time\_derivative), Dirichlet + Neumann BCs, transient.

\emph{Complex (coupled\_007):} ``A three-step cascade reaction A $\to$ B $\to$ C occurs in a 1D tube (length 10 cm). Species A (D $=$ 10$^{-4}$ m$^2$/s) enters from the left at 5 mol/m$^3$ and degrades at rate 1 s$^{-1}$ to produce B. Species B (D $=$ 5$\times$10$^{-5}$ m$^2$/s) degrades at rate 0.5 s$^{-1}$ to produce C. Species C (D $=$ 2$\times$10$^{-5}$ m$^2$/s) is collected at the right (concentration 0). All other boundaries: zero flux. Simulate 1 second.'' \textbf{GT:} 10 terms, 6 BCs, 3 variables, transient.

\section{Other Simulation Languages}
\label{app:crossdsl}

The reconstruction relies on three properties of the input language: the encoded problem is a sum of local contributions (compositional semantics), each construct maps to a weak-form operator (a direct operator map), and the request and the code can be stated in the same contract. \Cref{tab:adoption-roadmap} lists which frameworks have them.

\begin{table}[htbp]
\centering
\caption{Conditions for applying the reconstruction and comparison to other frameworks (\cref{sec:exp:ufl}).}
\label{tab:adoption-roadmap}
\small
\begin{tabular}{@{}lccc@{}}
\toprule
\textbf{Framework} & \textbf{Compositional semantics?} & \textbf{Direct op map?} & \textbf{Shared contract?} \\
\midrule
MOOSE            & Yes & Yes & Yes \\
FEniCS / UFL     & Yes & Yes, via symbolic forms & Yes \\
FreeFEM          & Yes & Yes, via \texttt{int2d} text & Yes \\
FiPy             & Yes & Yes, via term constructors & Yes \\
Devito           & Partial & Yes, via symbolic stencils & Yes \\
Firedrake        & Yes & Yes, via UFL frontend & Yes \\
deal.II          & Partial & Requires assembly analysis & Partial \\
OpenFOAM         & No direct weak-form exposure & Indirect & Requires different target \\
\bottomrule
\end{tabular}
\end{table}

\FloatBarrier

%% file: full_kernel_table.tex
\begingroup
\small
\setlength{\tabcolsep}{4pt}
\renewcommand{\arraystretch}{1.08}

\begin{longtable}{@{}>{\raggedright\arraybackslash}p{5cm}
                    >{\raggedright\arraybackslash}p{2.4cm}
                    >{\centering\arraybackslash}p{1.3cm}
                    >{\raggedright\arraybackslash}p{6.5cm}@{}}
\caption{Kernel--PDE weak-form mapping table for the covered MOOSE fragment; the 102 listed source-grounded weak-form contributions (from the 130-kernel runtime mapping table) cover 762 of 766 MooseBench ground-truth kernel-term instances (99.48\%). The ``Canonical Weak-Form Residual Contribution'' column records the most explicit kernel-level contribution to the assembled weak-form residual for each kernel, identified primarily from the MOOSE source code and supplemented by official documentation when clarification is needed. Depending on the kernel implementation, entries may appear in integration-by-parts form, as direct test-function-weighted strong-form terms, or, when necessary, as kernel-level energetic, matrix, stabilization, or discrete residual contributions. For covered entries, these source-grounded representations are normalized to a canonical weak-form operator descriptor $(\tau,\phi,\gamma)$, where $\tau$ denotes the operation on the trial-side quantity, $\phi$ the operation on the test function, and $\gamma$ the contraction pattern. This descriptor captures the weak-form shape associated with the corresponding normalized operator type; coefficients, coupled variables, constitutive choices, and nonlinear energetic factors are represented separately in the explicit contribution, mapping metadata, or the coefficient/material contract.}
\label{tab:full_kernel}
\\

\toprule
\textbf{MOOSE Kernel} & \textbf{Operator} & \textbf{Variable} & \textbf{Canonical Weak-Form Residual Contribution} \\
\midrule
\endfirsthead

\caption[]{Kernel--PDE mapping table for the covered MOOSE fragment (continued).}\\
\toprule
\textbf{MOOSE Kernel} & \textbf{Operator} & \textbf{Variable} & \textbf{Canonical Weak-Form Residual Contribution} \\
\midrule
\endhead

\midrule
\multicolumn{4}{r}{\scriptsize\itshape continued on next page} \\
\endfoot

\bottomrule
\endlastfoot

{\footnotesize\ttfamily ACGrGrMulti} & allen\_cahn & $\eta_i$ & $\left(\psi,\ \mu\!\left(\eta_i^3-\eta_i+2\eta_i\sum_j \gamma_{ij}\eta_j^2\right)\right)$ \\
{\footnotesize\ttfamily ACInterface} & allen\_cahn & $\eta_i$ & $(\kappa_i \nabla \eta_i,\ \nabla(L_i \psi))$ \\
{\footnotesize\ttfamily ACInterface2DMultiPhase1} & allen\_cahn & $\eta_{\alpha i}$ & $\left(\nabla (L \psi), \frac{1}{2} \frac {\partial \kappa} {\partial \nabla \eta_{\alpha i}} \sum (\nabla \eta_{\beta j})^2 \right)$ \\
{\footnotesize\ttfamily ACInterface2DMultiPhase2} & allen\_cahn & $\eta_{\alpha i}$ & $\left( \kappa \nabla \eta_{\alpha i}, \nabla (L \psi) \right)$ \\
{\footnotesize\ttfamily ACSwitching} & allen\_cahn & $\eta_i$ & $(\psi,\ \sum_j \frac{\partial h_j}{\partial \eta_i} F_j)$ \\
{\footnotesize\ttfamily ADBodyForce} & source & $u$ & $(\psi,\,-f)$ \\
{\footnotesize\ttfamily ADDiffusion} & diffusion & $u$ & $(\nabla \psi,\ \nabla u)$ \\
{\footnotesize\ttfamily ADHeatConduction} & diffusion & $T$ & $(\nabla \psi,\ k \nabla T)$ \\
{\footnotesize\ttfamily ADHeatConductionTimeDerivative} & time\_derivative & $T$ & $(\psi,\ \rho c_p \partial_t T)$ \\
{\footnotesize\ttfamily ADStressDivergenceRZTensors} & stress\_divergence & $(u_r,\,u_z)$ & $(\nabla \boldsymbol{\psi},\ \boldsymbol{\sigma}_{RZ}) + (\psi_r/r,\ \sigma_{\theta\theta})$ \\
{\footnotesize\ttfamily ADStressDivergenceTensors} & stress\_divergence & $\mathbf{u}$ & $(\nabla \boldsymbol{\psi},\ \boldsymbol{\sigma})$ \\
{\footnotesize\ttfamily ADTimeDerivative} & time\_derivative & $u$ & $(\psi,\ \partial_t u)$ \\
{\footnotesize\ttfamily ADVectorDiffusion} & diffusion & $\mathbf{u}$ & $(\nabla \boldsymbol{\psi},\ \nabla \mathbf{u})$ \\
{\footnotesize\ttfamily AllenCahn} & allen\_cahn & $\eta_i$ & $L\left(\dfrac{\partial f_{loc}}{\partial \eta_i} + \dfrac{\partial E_d}{\partial \eta_i},\ \psi_m\right)$ \\
{\footnotesize\ttfamily AnisotropicDiffusion} & diffusion & $u$ & $(\nabla \psi,\ \widetilde{k}\nabla u)$ \\
{\footnotesize\ttfamily ArrayBodyForce} & source & $\mathbf{u}$ & $(\boldsymbol{\psi},\ -\mathbf{f})$ \\
{\footnotesize\ttfamily ArrayDiffusion} & diffusion & $\mathbf{u}$ & $(\nabla \boldsymbol{\psi},\ \mathbf{D}\nabla \mathbf{u})$ \\
{\footnotesize\ttfamily ArrayReaction} & reaction & $\mathbf{u}$ & $(\boldsymbol{\psi},\ \mathbf{R}\mathbf{u})$ \\
{\footnotesize\ttfamily ArrayTimeDerivative} & time\_derivative & $\mathbf{u}$ & $(\boldsymbol{\psi},\ \mathbf{T}\,\partial_t \mathbf{u})$ \\
{\footnotesize\ttfamily BodyForce} & source & $u$ & $(\psi,\,-f)$ \\
{\footnotesize\ttfamily CHInterface} & cahn\_hilliard & $c_i$ & $(\kappa_i \nabla^2 c_i,\ \nabla \cdot (M_i \nabla \psi))$ \\
{\footnotesize\ttfamily CHMath} & other & $c$ & $(\nabla \psi,\ \nabla(c^3-c))$ \\
{\footnotesize\ttfamily CHSplitFlux} & cahn\_hilliard & $j$ & $(\psi,\ j + M\nabla \mu)$ \\
{\footnotesize\ttfamily CahnHilliard} & cahn\_hilliard & $c_i$ & $\left(M_i \left(\nabla \dfrac{\partial f_{loc}}{\partial c_i} + \nabla \dfrac{\partial E_d}{\partial c_i}\right),\ \nabla \psi\right)$ \\
{\footnotesize\ttfamily CoefDiffusion} & diffusion & $u$ & $(\nabla \psi,\ D \nabla u)$ \\
{\footnotesize\ttfamily CoefReaction} & reaction & $u$ & $(\psi,\ \lambda u)$ \\
{\footnotesize\ttfamily CoefTimeDerivative} & time\_derivative & $u$ & $(\psi,\ \alpha \partial_t u)$ \\
{\footnotesize\ttfamily ConservativeAdvection} & advection & $u$ & $-(\nabla \psi,\ \mathbf{v}u)$ \\
{\footnotesize\ttfamily CosseratStressDivergenceTensors} & stress\_divergence & $\mathbf{u}$ & $(\nabla \boldsymbol{\psi},\ \boldsymbol{\sigma}_{\mathrm{Cosserat}})$ \\
{\footnotesize\ttfamily CoupledBEEquilibriumSub} & other & $C_j$ & $(\psi,\ \partial_t(\phi \sum_i \nu_{ji} C_i))$ \\
{\footnotesize\ttfamily CoupledBEKinetic} & other & $C_m$ & $(\psi,\ \partial_t(\phi \sum_m w_m C_m))$ \\
{\footnotesize\ttfamily CoupledConvectionReactionSub} & reaction & $C_j$ & $(\psi,\ \mathbf{q}\cdot \nabla(\sum_i \nu_{ji} C_i))$ \\
{\footnotesize\ttfamily CoupledDiffusionReactionSub} & diffusion & $C_j$ & $(\nabla \psi,\ D \nabla (\sum_i \nu_{ji} C_i))$ \\
{\footnotesize\ttfamily CoupledForce} & source & $u$ & $(\psi,\ -\alpha v)$ \\
{\footnotesize\ttfamily CoupledSwitchingTimeDerivative} & time\_derivative & $\eta_i$ & $\left(\psi,\ \left(\sum_j \frac{\partial h_j}{\partial \eta_i} F_j\right)\partial_t v\right)$ \\
{\footnotesize\ttfamily CoupledTimeDerivative} & time\_derivative & $u$ & $(\psi,\ \partial_t v)$ \\
{\footnotesize\ttfamily CurlCurlField} & curl\_curl & $\mathbf{u}$ & $(\nabla \times \boldsymbol{\psi},\ a\,\nabla \times \mathbf{u})$ \\
{\footnotesize\ttfamily DarcyFluxPressure} & other & $P$ & $(\nabla \psi,\ \dfrac{K}{\mu}(\nabla P-\rho \mathbf{g}))$ \\
{\footnotesize\ttfamily DesorptionFromMatrix} & source & $m$ & $(\psi,\ \dot{m})$ \\
{\footnotesize\ttfamily DesorptionToPorespace} & source & $m$ & $(\psi,\ -\dot{m})$ \\
{\footnotesize\ttfamily Diffusion} & diffusion & $u$ & $(\nabla \psi,\ \nabla u)$ \\
{\footnotesize\ttfamily DynamicStressDivergenceTensors} & stress\_divergence & $\mathbf{u}$ & $(\nabla \boldsymbol{\psi},\ \boldsymbol{\sigma}_{\mathrm{dyn}})$ \\
{\footnotesize\ttfamily FunctionDiffusion} & diffusion & $u$ & $(\nabla \psi,\ f(\mathbf{x},t)\nabla v)$ \\
{\footnotesize\ttfamily FunctorKernel} & other & $u$ & $(\psi,\ \pm (p-u))$ \\
{\footnotesize\ttfamily GradientComponent} & other & $u$ & $(\psi,\ u - \partial_\alpha v)$ \\
{\footnotesize\ttfamily Gravity} & source & $\mathbf{u}$ & $(\boldsymbol{\psi},\ -\rho\,\mathbf{g})$ \\
{\footnotesize\ttfamily HeatConduction} & diffusion & $T$ & $(\nabla \psi,\ k \nabla T)$ \\
{\footnotesize\ttfamily HeatConductionTimeDerivative} & time\_derivative & $T$ & $(\psi,\ \rho C_p \partial_t T)$ \\
{\footnotesize\ttfamily HeatSource} & source & $T$ & $(\psi,\ -\dot{q})$ \\
{\footnotesize\ttfamily INSADMass} & ns\_continuity & $p$ & $(\psi,\ -\nabla \cdot \mathbf{u})$ \\
{\footnotesize\ttfamily INSADMomentumAdvection} & advection & $\mathbf{u}$ & $(\boldsymbol{\psi},\ \rho (\mathbf{u}\cdot\nabla)\mathbf{u})$ \\
{\footnotesize\ttfamily INSADMomentumPressure} & ns\_pressure & $\mathbf{u}$ & $-(\nabla \cdot \boldsymbol{\psi},\ p)$ \\
{\footnotesize\ttfamily INSADMomentumSUPG} & ns\_viscous & $\mathbf{u}$ & $(\tau\,\mathbf{u}\cdot\nabla \boldsymbol{\psi},\ \mathbf{R}_{\mathrm{mom}})$ \\
{\footnotesize\ttfamily INSADMomentumViscous} & ns\_viscous & $\mathbf{u}$ & $(\nabla \boldsymbol{\psi},\ \boldsymbol{\tau})$ \\
{\footnotesize\ttfamily INSFEFluidMassKernel} & ns\_continuity & $p$ & $-\rho \,\bar{\mathbf{v}} \cdot \nabla \psi$ \\
{\footnotesize\ttfamily INSFEFluidMomentumKernel} & ns\_viscous & $\mathbf{v}$ & $\rho\,\epsilon\,(\mathbf{v}\cdot\nabla\mathbf{v})\,\psi+\epsilon\,(\nabla p)\,\psi-\epsilon\rho\mathbf{g}\,\psi+\boldsymbol{\tau}\cdot\nabla\psi+\mathbf{f}_{\mathrm{fric}}\,\psi$ \\
{\footnotesize\ttfamily INSMass} & ns\_continuity & $p$ & $(\psi,\ -\nabla \cdot \mathbf{u})$ \\
{\footnotesize\ttfamily INSMomentumLaplaceForm} & ns\_viscous & $\mathbf{u}$ & $(\nabla \boldsymbol{\psi},\ \mu \nabla \mathbf{u})$ \\
{\footnotesize\ttfamily INSMomentumTimeDerivative} & time\_derivative & $\mathbf{u}$ & $(\boldsymbol{\psi},\ \rho\,\partial_t \mathbf{u})$ \\
{\footnotesize\ttfamily InertialForce} & inertia & $\mathbf{u}$ & $(\boldsymbol{\psi},\ \rho\,\ddot{\mathbf{u}}+\eta\,\rho\,\dot{\mathbf{u}})$ \\
{\footnotesize\ttfamily InertialTorque} & inertia & $\mathbf{u}$ & $(\psi,\ \rho\,\epsilon_{ijk}u_j\ddot{u}_k)$ \\
{\footnotesize\ttfamily KKSMultiACBulkC} & allen\_cahn & $\eta_1$ & $\left(\psi,\ -\dfrac{\partial F_1}{\partial c_1}\sum_j \dfrac{\partial h_j}{\partial \eta_1} c_j\right)$ \\
{\footnotesize\ttfamily KKSMultiACBulkF} & allen\_cahn & $\eta_i$ & $\left(\psi,\ \sum_j \dfrac{\partial h_j}{\partial \eta_i}F_j + w_i\dfrac{\partial g_i}{\partial \eta_i}\right)$ \\
{\footnotesize\ttfamily KKSPhaseChemicalPotential} & other & $c_a$ & $\left(\psi,\ \dfrac{\partial F_a}{\partial c_a}\dfrac{1}{k_a}-\dfrac{\partial F_b}{\partial c_b}\dfrac{1}{k_b}\right)$ \\
{\footnotesize\ttfamily MassEigenKernel} & other & $u$ & $\lambda(\psi,\ -c\,u)$ \\
{\footnotesize\ttfamily MassLumpedTimeDerivative} & time\_derivative & $u$ & $(\psi,\ \partial_t u)$ \\
{\footnotesize\ttfamily MatDiffusion} & diffusion & $u$ & $(\nabla \psi,\ D\nabla u)$ \\
{\footnotesize\ttfamily MatReaction} & reaction & $u$ & $(\psi,\ -L\,v)$ \\
{\footnotesize\ttfamily NestedKKSMultiACBulkC} & allen\_cahn & $\eta_i$ & $\left(\psi,\ -\sum_m \frac{\partial F_1}{\partial c_{m,1}} \sum_j \frac{\partial h_j}{\partial \eta_i} c_{m,j}\right)$ \\
{\footnotesize\ttfamily NestedKKSMultiACBulkF} & allen\_cahn & $\eta_i$ & $\left(\psi,\ \sum_j \dfrac{\partial h_j}{\partial \eta_i}F_j + w_i\dfrac{\partial g_i}{\partial \eta_i}\right)$ \\
{\footnotesize\ttfamily NullKernel} & other & $u$ & $(\psi,\ 0)$ \\
{\footnotesize\ttfamily PINSFEFluidPressureTimeDerivative} & time\_derivative & $p$ & $\left(\psi,\ \epsilon \left(\dfrac{\partial \rho}{\partial T}\dot{T}+\dfrac{\partial \rho}{\partial p}\dot{p}\right)\right)$ \\
{\footnotesize\ttfamily PINSFEFluidVelocityTimeDerivative} & time\_derivative & $v_i$ & $\left(\psi,\ \rho \dfrac{\partial v_i}{\partial t} + v_i \dfrac{\partial \rho}{\partial t}\right)$ \\
{\footnotesize\ttfamily PorousFlowAdvectiveFlux} & advection & $\chi_\beta^\kappa$ & $\left(\nabla\psi,\ \sum_{\beta}\chi_\beta^\kappa \rho_\beta \dfrac{k\,k_{r,\beta}}{\mu_\beta}\left(\nabla P_\beta-\rho_\beta \mathbf{g}\right)\right)$ \\
{\footnotesize\ttfamily PorousFlowBasicAdvection} & advection & $u$ & $-(\nabla\psi,\ \mathbf{v}_\beta u)$ \\
{\footnotesize\ttfamily PorousFlowDispersiveFlux} & diffusion & $\chi_\beta^\kappa$ & $\left(\nabla\psi,\ \sum_{\beta}\rho_\beta \mathbf{D}_\beta^\kappa \nabla X_\beta^\kappa\right)$ \\
{\footnotesize\ttfamily PorousFlowEffectiveStressCoupling} & pf\_effective\_stress & $\mathbf{u}$ & $-\alpha_B\,p_{\mathrm{eff}}\,\partial_{x_c}\psi$ \\
{\footnotesize\ttfamily PorousFlowEnergyTimeDerivative} & time\_derivative & $T$ & $\left(\psi,\ \partial_t\left((1-\phi)e_r+\phi\sum_{\beta}S_\beta \rho_\beta e_\beta\right)\right)$ \\
{\footnotesize\ttfamily PorousFlowFluxLimitedTVDAdvection} & advection & $u$ & $\left(\nabla\psi,\ \sum_{\beta}\chi_\beta^\kappa \rho_\beta \dfrac{k\,k_{r,\beta}}{\mu_\beta}\left(\nabla P_\beta-\rho_\beta \mathbf{g}\right)\right)$ \\
\kernelname{PorousFlowFullySaturatedAdvectiveFlux} & advection & $\chi^\kappa$ & $\left(\nabla\psi,\ (\rho)\chi^\kappa \dfrac{k}{\mu}\left(\nabla P-\rho \mathbf{g}\right)\right)$ \\
{\footnotesize\ttfamily PorousFlowFullySaturatedDarcyBase} & pf\_darcy\_flux & $P$ & $\left(\nabla\psi,\ (\rho)\dfrac{k}{\mu}\left(\nabla P-\rho \mathbf{g}\right)\right)$ \\
{\footnotesize\ttfamily PorousFlowFullySaturatedDarcyFlow} & advection & $\chi^\kappa$ & $\left(\nabla\psi,\ (\rho)\chi^\kappa \dfrac{k}{\mu}\left(\nabla P-\rho \mathbf{g}\right)\right)$ \\
{\footnotesize\ttfamily PorousFlowHeatAdvection} & advection & $h$ & $\left(\nabla\psi,\ \sum_{\beta} h_\beta \rho_\beta \dfrac{k\,k_{r,\beta}}{\mu_\beta}\left(\nabla P_\beta-\rho_\beta \mathbf{g}\right)\right)$ \\
{\footnotesize\ttfamily PorousFlowHeatConduction} & diffusion & $T$ & $(\nabla\psi,\ \lambda \nabla T)$ \\
{\footnotesize\ttfamily PorousFlowHeatMassTransfer} & other & $u$ & $(\psi,\ k_{\mathrm{tr}}(u-v))$ \\
{\footnotesize\ttfamily PorousFlowMassTimeDerivative} & time\_derivative & $\chi^\kappa$ & $\left(\psi,\ \partial_t\!\left(\phi \sum_{\beta} (\rho_\beta) S_\beta \chi_\beta^\kappa\right)\right)$ \\
{\footnotesize\ttfamily PorousFlowMassVolumetricExpansion} & source & $\chi^\kappa$ & $\left(\psi,\ \phi \sum_{\beta} (\rho_\beta) S_\beta \chi_\beta^\kappa \,\dot{\epsilon}_v\right)$ \\
{\footnotesize\ttfamily PorousFlowPreDis} & reaction & $c_i$ & $\left(\psi,\ \phi\,S_{\mathrm{aq}}\sum_r \nu_{ir}\rho_r R_r\right)$ \\
{\footnotesize\ttfamily PrimaryConvection} & advection & $u$ & $(\psi,\ \mathbf{v}_D\cdot\nabla u)$ \\
{\footnotesize\ttfamily PrimaryDiffusion} & diffusion & $c_j$ & $(\nabla\psi,\ D \nabla c_j)$ \\
{\footnotesize\ttfamily PrimaryTimeDerivative} & time\_derivative & $c_j$ & $\left(\psi,\ \phi \dfrac{\partial c_j}{\partial t}\right)$ \\
{\footnotesize\ttfamily Reaction} & reaction & $u$ & $(\psi,\ \lambda u)$ \\
{\footnotesize\ttfamily SplitCHParsed} & cahn\_hilliard & $c$ & $\left(\psi,\ \dfrac{\partial f_{loc}}{\partial c} + \dfrac{\partial E_d}{\partial c} - w\right) + (\nabla\psi,\ \kappa \nabla c)$ \\
{\footnotesize\ttfamily SplitCHWRes} & cahn\_hilliard & $w$ & $(M\nabla w,\nabla\psi)$ \\
{\footnotesize\ttfamily StressDivergenceTensors} & stress\_divergence & $\mathbf{u}$ & $(\nabla\boldsymbol{\psi},\ \boldsymbol{\sigma})$ \\
{\footnotesize\ttfamily StressDivergenceTensorsTruss} & stress\_divergence & $\mathbf{u}$ & $(\nabla\boldsymbol{\psi},\ \boldsymbol{\sigma})$ \\
{\footnotesize\ttfamily SusceptibilityTimeDerivative} & time\_derivative & $u$ & $\left(\psi,\ F(u,a,b,\ldots)\,\partial_t u\right)$ \\
{\footnotesize\ttfamily TimeDerivative} & time\_derivative & $u$ & $(\psi,\ \partial_t u)$ \\
{\footnotesize\ttfamily VectorBodyForce} & source & $\mathbf{u}$ & $(\boldsymbol{\psi},\ -\mathbf{f})$ \\
{\footnotesize\ttfamily VectorDiffusion} & diffusion & $\mathbf{u}$ & $(\nabla\boldsymbol{\psi},\ \nabla \mathbf{u})$ \\
{\footnotesize\ttfamily VectorFunctionReaction} & reaction & $\mathbf{u}$ & $(\boldsymbol{\psi},\ \lambda(\mathbf{x},t)\mathbf{u})$ \\
{\footnotesize\ttfamily VectorTimeDerivative} & time\_derivative & $\mathbf{u}$ & $(\boldsymbol{\psi},\ \partial_t \mathbf{u})$ \\

\end{longtable}
\endgroup